\documentclass[runningheads]{llncs}

\usepackage[utf8]{inputenc}
\usepackage[T1]{fontenc}

\usepackage{pgfplots}
\usepackage{pgfplotstable}
\usepackage{amsmath}
\usepackage{amssymb}
\usepackage{amsfonts}
\usepackage{tikz}
\usepackage{xcolor}
\usepackage{algorithm}
\usepackage[noend]{algpseudocode}
\usepackage{mathtools}
\usepackage{booktabs}
\usepackage{url}
\usepackage{cite}
\usepackage{array}
\usepackage{placeins}
\usepackage[hidelinks, bookmarksdepth=2, bookmarksnumbered, bookmarksopen]{hyperref}
\usepgfplotslibrary{fillbetween}
\usepgfplotslibrary{groupplots}
\pgfplotsset{compat=1.18}
\usetikzlibrary{positioning, arrows.meta}
\usetikzlibrary{calc}
\usetikzlibrary{fit}
\usetikzlibrary{patterns}
\usetikzlibrary{automata}
\definecolor{oracle_blue}{HTML}{0072B2}   %
\definecolor{unknown_skyblue}{HTML}{56B4E9} %
\definecolor{tcirl_green}{HTML}{009E73}   %
\definecolor{tcirl_annotation}{HTML}{00664A} %
\definecolor{dqn_orange}{HTML}{D55E00}    %
\definecolor{dqn_amber}{HTML}{E69F00}     %
\definecolor{unknown_red}{HTML}{CC3333}
\definecolor{counterfactual_purple}{HTML}{CC79A7} %
\newlength{\plotAxisWidth}
\newlength{\plotAxisHeight}
\AtBeginDocument{%
   \setlength{\plotAxisWidth}{0.5\textwidth}%
   \addtolength{\plotAxisWidth}{-42pt}%
   \setlength{\plotAxisHeight}{0.65\plotAxisWidth}%
}

\pgfplotsset{
   paper plot/.style={
         scale only axis=true,
         width=\plotAxisWidth,
         height=\plotAxisHeight,
         grid=major,
         tick label style={font=\small},
         label style={font=\small},
         yticklabel style={text width=14pt, align=right},
      },
   paper legend/.style={
         legend style={
               at={(0.5,1.05)},
               anchor=south,
               draw=none,
               font=\scriptsize,
               legend columns=-1,
               legend cell align=left,
               inner sep=1pt,
               column sep=8pt,
            },
      },
}

\pgfplotsset{
   eval time axis/.style={
         xlabel={Time in Eval Episode},
         xtick={0,20,...,100},
      },
}

\newcommand{\defineSeriesStyle}[4]{%
   \expandafter\def\csname seriesColor@#1\endcsname{#2}%
   \pgfplotsset{%
      #1/line/.style={color=#2, thick, #3, no marks},
      #1/mark/.style={color=#2, thick, solid, #4},
      #1/fill/.style={fill=#2, fill opacity=0.15, draw=none},
   }%
}

\newcommand{\seriesColor}[1]{\csname seriesColor@#1\endcsname}

\defineSeriesStyle{oracle}{oracle_blue}{dashed}{mark=*, mark size=1.2pt, mark repeat=25}
\defineSeriesStyle{unknown}{unknown_skyblue}{dotted}{mark=triangle*, mark size=1.5pt, mark repeat=25}
\defineSeriesStyle{tcirl}{tcirl_green}{solid}{mark=square*, mark size=1.2pt, mark repeat=25}
\defineSeriesStyle{dqn4}{dqn_orange}{densely dashdotted}{mark=diamond*, mark size=1.5pt, mark repeat=25}
\defineSeriesStyle{dqn1}{dqn_amber}{dashdotted}{mark=pentagon*, mark size=1.5pt, mark repeat=25}
\defineSeriesStyle{exacttransfer}{black}{dashed}{mark=+, mark size=1.5pt, mark repeat=25}
\defineSeriesStyle{notransfer}{black}{densely dotted}{mark=x, mark size=1.5pt, mark repeat=25}
\defineSeriesStyle{counterfactual}{counterfactual_purple}{solid}{mark=square, mark size=1.2pt, mark repeat=25}
\defineSeriesStyle{oraclecounterfactual}{counterfactual_purple}{dashed}{mark=o, mark size=1.2pt, mark repeat=25}

\newcommand{\dropLegendEntry}[1]{}

\newcommand{\addTrajSeries}[4]{%
   \addplot[draw=none, forget plot, name path=t#1lo] table [x=timestep, y=ci_lo] {#2};
   \addplot[draw=none, forget plot, name path=t#1hi] table [x=timestep, y=ci_hi] {#2};
   \addplot[#1/#3] table [x=timestep, y=mean] {#2};
   \addlegendentry{#4}
   \addplot[#1/fill, forget plot] fill between[of=t#1lo and t#1hi];
}

\newcommand{\addTrainSeries}[4]{%
   \addplot[draw=none, forget plot, name path=m#1lo] table [x=step, y=ci_lo] {#2};
   \addplot[draw=none, forget plot, name path=m#1hi] table [x=step, y=ci_hi] {#2};
   \addplot[#1/#3] table [x=step, y=mean] {#2};
   \addlegendentry{#4}
   \addplot[#1/fill, forget plot] fill between[of=m#1lo and m#1hi];
}

\newcommand{\addSolvedSeries}[4]{%
   \addplot[draw=none, forget plot, name path=s#1lo] table [x=step, y expr=100*\thisrow{ci_lo}] {#2};
   \addplot[draw=none, forget plot, name path=s#1hi] table [x=step, y expr=100*\thisrow{ci_hi}] {#2};
   \addplot[#1/#3] table [x=step, y expr=100*\thisrow{mean}] {#2};
   \addlegendentry{#4}
   \addplot[#1/fill, forget plot] fill between[of=s#1lo and s#1hi];
}

\newcommand{\createTrainingPlot}[9]{%
   \pgfplotstableread[col sep=comma]{#1}\datatableA
   \pgfplotstableread[col sep=comma]{#2}\datatableB
   \pgfplotstableread[col sep=comma]{#3}\datatableC
   \pgfplotstableread[col sep=comma]{#4}\datatableD
   \pgfplotstableread[col sep=comma]{#5}\datatableE
   \begin{tikzpicture}
      \begin{axis}[
            xlabel={Training Steps},
            ylabel={#6},
            paper plot,
            ymin=#7,
            ymax=#8,
            scaled x ticks=base 10:-3,
            xtick scale label code/.code={},
            xticklabel={\pgfmathprintnumber[fixed,precision=0]{\tick}k},
            #9,
         ]
         \addTrainSeries{tcirl}{\datatableC}{line}{TCIRL (ours)}
         \addTrainSeries{dqn4}{\datatableD}{line}{DQN (4)}
         \addTrainSeries{dqn1}{\datatableE}{line}{DQN (1)}
         \addTrainSeries{unknown}{\datatableB}{line}{Q-learning (cause unknown)}
         \addTrainSeries{oracle}{\datatableA}{line}{Q-learning (cause known)}
      \end{axis}
   \end{tikzpicture}%
}

\newcommand{\createTrajectoryPlot}[8]{%
   \pgfplotstableread[col sep=comma]{#1}\trajdataA
   \pgfplotstableread[col sep=comma]{#2}\trajdataB
   \pgfplotstableread[col sep=comma]{#3}\trajdataC
   \pgfplotstableread[col sep=comma]{#4}\trajdataD
   \pgfplotstableread[col sep=comma]{#5}\trajdataE
   \begin{tikzpicture}
      \begin{axis}[
            eval time axis,
            ylabel={#6},
            paper plot,
            ymin=#7,
            ymax=#8,
            every axis legend/.code={\let\addlegendentry\dropLegendEntry}
         ]
         \addTrajSeries{tcirl}{\trajdataC}{line}{TCIRL (ours)}
         \addTrajSeries{dqn4}{\trajdataD}{line}{DQN (4)}
         \addTrajSeries{dqn1}{\trajdataE}{line}{DQN (1)}
         \addTrajSeries{unknown}{\trajdataB}{line}{Q-learning (cause unknown)}
         \addTrajSeries{oracle}{\trajdataA}{line}{Q-learning (cause known)}
      \end{axis}
   \end{tikzpicture}%
}

\newcommand{\createIntroSymptomPlot}[3][]{%
   \pgfplotstableread[col sep=comma]{#2}\introUnknown
   \pgfplotstableread[col sep=comma]{#3}\introTcirl
   \begin{tikzpicture}[#1]
      \begin{axis}[
            eval time axis,
            ylabel={Symptom Level $\symptomLevel$},
            paper plot,
            ymin=0, ymax=11,
            ytick={0,2,...,8},
            paper legend,
            legend style={legend columns=1},
         ]
         \addTrajSeries{unknown}{\introUnknown}{mark}{Cause unknown}
         \addTrajSeries{tcirl}{\introTcirl}{mark}{Cause inferred (ours)}
         \draw[tcirl_annotation, thin, dashed] (axis cs:10,0) -- (axis cs:10,9);
         \node[anchor=west, font=\scriptsize, tcirl_annotation] (p2) at (axis cs:40,1.1) {Phase 2};
         \draw[tcirl_annotation, thin, -{Stealth}] (axis cs:6,1.1) -- (p2);
         \draw[unknown_red, thin, dashed, {Stealth}-{Stealth}] (axis cs:20,9.5) -- (axis cs:70,9.5);
         \node[anchor=south, font=\scriptsize, unknown_red, inner sep=1pt] at (axis cs:45,9.6) {Mostly in Phase 1};
      \end{axis}
   \end{tikzpicture}%
}

\newcommand{\acceptRead}[6]{%
   \pgfplotstablegetelem{0}{#2}\of#1\edef#3{\pgfplotsretval}%
   \pgfplotstablegetelem{0}{#2_hi}\of#1\edef#6{\pgfplotsretval}%
   \pgfmathsetmacro{#4}{#6-#3}%
   \pgfplotstablegetelem{0}{#2_lo}\of#1\pgfmathsetmacro{#5}{#3-\pgfplotsretval}%
}

\newcommand{\acceptBar}[6]{%
   \addplot[#1, bar shift=0pt, error bars/.cd, y dir=both, y explicit, error bar style={black, thin}, error mark options={rotate=90, mark size=2pt}] coordinates {(#2, #3) += (0, #4) -= (0, #5)};
   \addplot[draw=none, fill=none, bar shift=0pt, forget plot, point meta=explicit, nodes near coords={\pgfmathprintnumber[fixed, fixed zerofill, precision=1]{\pgfplotspointmeta}}] coordinates {(#2, #6) [#3]};
}
\newcommand{\createAcceptGroupPlot}[6]{%
   \pgfplotstableread[col sep=comma]{#1}\acceptdataL
   \acceptRead{\acceptdataL}{oracle}{\aLoracle}{\aLoracleP}{\aLoracleM}{\aLoracleH}
   \acceptRead{\acceptdataL}{unknown}{\aLunknown}{\aLunknownP}{\aLunknownM}{\aLunknownH}
   \acceptRead{\acceptdataL}{tcirl}{\aLtcirl}{\aLtcirlP}{\aLtcirlM}{\aLtcirlH}
   \acceptRead{\acceptdataL}{dqn4}{\aLdqnfour}{\aLdqnfourP}{\aLdqnfourM}{\aLdqnfourH}
   \acceptRead{\acceptdataL}{dqn1}{\aLdqnone}{\aLdqnoneP}{\aLdqnoneM}{\aLdqnoneH}
   \pgfplotstableread[col sep=comma]{#3}\acceptdataR
   \acceptRead{\acceptdataR}{oracle}{\aRoracle}{\aRoracleP}{\aRoracleM}{\aRoracleH}
   \acceptRead{\acceptdataR}{unknown}{\aRunknown}{\aRunknownP}{\aRunknownM}{\aRunknownH}
   \acceptRead{\acceptdataR}{tcirl}{\aRtcirl}{\aRtcirlP}{\aRtcirlM}{\aRtcirlH}
   \acceptRead{\acceptdataR}{dqn4}{\aRdqnfour}{\aRdqnfourP}{\aRdqnfourM}{\aRdqnfourH}
   \acceptRead{\acceptdataR}{dqn1}{\aRdqnone}{\aRdqnoneP}{\aRdqnoneM}{\aRdqnoneH}
   \begin{tikzpicture}
      \begin{groupplot}[
            group style={
                  group size=2 by 1,
                  horizontal sep=0pt,
                  ylabels at=edge left,
               },
            ybar,
            /pgf/bar width=14pt,
            scale only axis=true,
            width=\plotAxisWidth,
            height=\plotAxisHeight,
            ymin=0,
            symbolic x coords={tcirl,unknown,oracle,dqn4,dqn1},
            xtick=data,
            xticklabels={},
            ytick={0,50,100},
            yticklabel style={text width=14pt, align=right},
            tick label style={font=\small},
            label style={font=\small},
            every node near coord/.append style={font=\scriptsize\bfseries, anchor=south, yshift=1pt},
            enlarge x limits=0.25,
            clip=false,
            ymax=#6,
         ]
         \nextgroupplot[
            ylabel={#5},
            ylabel style={font=\small},
            title={#2},
            title style={at={(0.5,0)}, anchor=north, yshift=-8pt, font=\footnotesize},
            legend to name=acceptlegend,
            paper legend,
            legend style={legend columns=3},
         ]
         \acceptBar{fill=tcirl_green, draw=tcirl_green}{tcirl}{\aLtcirl}{\aLtcirlP}{\aLtcirlM}{\aLtcirlH}
         \addlegendentry{TCIRL}
         \acceptBar{pattern=dots, pattern color=unknown_skyblue, draw=unknown_skyblue}{unknown}{\aLunknown}{\aLunknownP}{\aLunknownM}{\aLunknownH}
         \addlegendentry{Q-learning (cause unknown)}
         \acceptBar{pattern=north east lines, pattern color=oracle_blue, draw=oracle_blue}{oracle}{\aLoracle}{\aLoracleP}{\aLoracleM}{\aLoracleH}
         \addlegendentry{Q-learning (cause known)}
         \acceptBar{pattern=grid, pattern color=dqn_orange, draw=dqn_orange}{dqn4}{\aLdqnfour}{\aLdqnfourP}{\aLdqnfourM}{\aLdqnfourH}
         \addlegendentry{DQN (4)}
         \acceptBar{pattern=vertical lines, pattern color=dqn_amber, draw=dqn_amber}{dqn1}{\aLdqnone}{\aLdqnoneP}{\aLdqnoneM}{\aLdqnoneH}
         \addlegendentry{DQN (1)}
         \nextgroupplot[
            yticklabels={},
            title={#4},
            title style={at={(0.5,0)}, anchor=north, yshift=-8pt, font=\footnotesize},
         ]
         \acceptBar{fill=tcirl_green, draw=tcirl_green}{tcirl}{\aRtcirl}{\aRtcirlP}{\aRtcirlM}{\aRtcirlH}
         \acceptBar{pattern=dots, pattern color=unknown_skyblue, draw=unknown_skyblue}{unknown}{\aRunknown}{\aRunknownP}{\aRunknownM}{\aRunknownH}
         \acceptBar{pattern=north east lines, pattern color=oracle_blue, draw=oracle_blue}{oracle}{\aRoracle}{\aRoracleP}{\aRoracleM}{\aRoracleH}
         \acceptBar{pattern=grid, pattern color=dqn_orange, draw=dqn_orange}{dqn4}{\aRdqnfour}{\aRdqnfourP}{\aRdqnfourM}{\aRdqnfourH}
         \acceptBar{pattern=vertical lines, pattern color=dqn_amber, draw=dqn_amber}{dqn1}{\aRdqnone}{\aRdqnoneP}{\aRdqnoneM}{\aRdqnoneH}
      \end{groupplot}
      \node[above] at ($(group c1r1.north)!0.5!(group c2r1.north)+(0,0.2cm)$) {\ref{acceptlegend}};
   \end{tikzpicture}%
}

\newcommand{\createInferenceComparisonPlot}[5]{%
   \pgfplotstableread[col sep=comma]{#1}\geneinfdatatable
   \pgfplotstableread[col sep=comma]{#2}\trafficinfdatatable
   \begin{tikzpicture}
      \begin{axis}[
            xlabel={Training Steps},
            ylabel={#3},
            paper plot,
            ymin=#4,
            ymax=#5,
            scaled x ticks=base 10:-3,
            xtick scale label code/.code={},
            xticklabel={\pgfmathprintnumber[fixed,precision=0]{\tick}k},
            paper legend,
         ]
         \addplot[draw=none, forget plot, name path=geneinflo] table [x=step, y=ci_lo] {\geneinfdatatable};
         \addplot[draw=none, forget plot, name path=geneinfhi] table [x=step, y=ci_hi] {\geneinfdatatable};
         \addplot[oracle/line] table [x=step, y=mean] {\geneinfdatatable};
         \addlegendentry{Genetic therapy}
         \addplot[oracle/fill, forget plot] fill between[of=geneinflo and geneinfhi];

         \addplot[draw=none, forget plot, name path=trafficinflo] table [x=step, y=ci_lo] {\trafficinfdatatable};
         \addplot[draw=none, forget plot, name path=trafficinfhi] table [x=step, y=ci_hi] {\trafficinfdatatable};
         \addplot[tcirl/line] table [x=step, y=mean] {\trafficinfdatatable};
         \addlegendentry{Traffic signal}
         \addplot[tcirl/fill, forget plot] fill between[of=trafficinflo and trafficinfhi];
      \end{axis}
   \end{tikzpicture}%
}

\newcommand{\createSolvedPlot}[7]{%
   \pgfplotstableread[col sep=comma]{#1}\solvedA
   \pgfplotstableread[col sep=comma]{#2}\solvedB
   \pgfplotstableread[col sep=comma]{#3}\solvedC
   \pgfplotstableread[col sep=comma]{#4}\solvedD
   \pgfplotstableread[col sep=comma]{#5}\solvedE
   \begin{tikzpicture}
      \begin{axis}[
            xlabel={Training Steps},
            ylabel={#6},
            paper plot,
            ymin=0,
            ymax=105,
            scaled x ticks=base 10:-3,
            xtick scale label code/.code={},
            xticklabel={\pgfmathprintnumber[fixed,precision=0]{\tick}k},
            #7
         ]
         \addSolvedSeries{tcirl}{\solvedC}{line}{TCIRL (ours)}
         \addSolvedSeries{dqn4}{\solvedD}{line}{DQN (4)}
         \addSolvedSeries{dqn1}{\solvedE}{line}{DQN (1)}
         \addSolvedSeries{unknown}{\solvedB}{line}{Q-learning (cause unknown)}
         \addSolvedSeries{oracle}{\solvedA}{line}{Q-learning (cause known)}
      \end{axis}
   \end{tikzpicture}%
}

\newcommand{\addSeedColumn}[3]{%
   \addplot[only marks, mark=*, mark size=0.6pt, color=\seriesColor{#2}, fill opacity=0.8]
   table [x expr={#1 + 0.065*(mod(\coordindex,7)-3)}, y=#2] {#3};
}

\newcommand{\createSeedScorePlot}[5]{%
   \pgfplotstableread[col sep=comma]{#1}\seedScoreTable
   \begin{tikzpicture}
      \begin{axis}[
            ylabel={#2},
            paper plot,
            ymin=#3,
            ymax=#4,
            xmin=0.5,
            xmax=5.5,
            xtick={1,2,3,4,5},
            xticklabels={Q (known), Q (unknown), TCIRL, DQN (4), DQN (1)},
            x tick label style={font=\scriptsize, rotate=25, anchor=north east},
            grid=major,
            #5
         ]
         \addSeedColumn{1}{oracle}{\seedScoreTable}
         \addSeedColumn{2}{unknown}{\seedScoreTable}
         \addSeedColumn{3}{tcirl}{\seedScoreTable}
         \addSeedColumn{4}{dqn4}{\seedScoreTable}
         \addSeedColumn{5}{dqn1}{\seedScoreTable}
      \end{axis}
   \end{tikzpicture}%
}

\newcommand{\cyclesReturn}[3]{%
   $\pgfmathprintnumber[fixed, fixed zerofill, precision=1]{#1}$ & {\scriptsize$[\pgfmathprintnumber[fixed, fixed zerofill, precision=1]{#2}, \pgfmathprintnumber[fixed, fixed zerofill, precision=1]{#3}]$}%
}

\newcommand{\cyclesAction}[1]{\if a#1\actAdvance\else\actHold\fi}
\def\cyclesLetters#1{\ifx\relax#1\else\cyclesAction{#1}\expandafter\cyclesLetters\fi}

\newcommand{\cyclesExactRow}[3]{%
   \pgfplotstablegetelem{#1}{word}\of\cyclesExact\xdef\cyclesWord{\pgfplotsretval}%
   \pgfplotstablegetelem{#1}{exact_return}\of\cyclesExact\xdef\cyclesMean{\pgfplotsretval}%
   #2 $\expandafter\cyclesLetters\cyclesWord\relax$ & #3 & $\pgfmathprintnumber[fixed, fixed zerofill, precision=1]{\cyclesMean}$ &%
}

\newcommand{\cyclesOptimalRow}[3]{%
   \pgfplotstablegetelem{#1}{exact_return}\of\cyclesExact\xdef\cyclesMean{\pgfplotsretval}%
   #2 & #3 & $\pgfmathprintnumber[fixed, fixed zerofill, precision=1]{\cyclesMean}$ &%
}

\newcommand{\cyclesLearnedRow}[3]{%
   \pgfplotstablegetelem{0}{#1}\of\cyclesFinal\xdef\cyclesMean{\pgfplotsretval}%
   \pgfplotstablegetelem{0}{#1_lo}\of\cyclesFinal\xdef\cyclesLo{\pgfplotsretval}%
   \pgfplotstablegetelem{0}{#1_hi}\of\cyclesFinal\xdef\cyclesHi{\pgfplotsretval}%
   #2 & #3 & \cyclesReturn{\cyclesMean}{\cyclesLo}{\cyclesHi}%
}

\newcommand{\createCyclesTable}{%
   \pgfplotstableread[col sep=comma]{results/traffic_signal/cycles_optimum.csv}\cyclesExact
   \pgfplotstableread[col sep=comma]{results/traffic_signal/final_returns.csv}\cyclesFinal
   \small
   \setlength{\tabcolsep}{4pt}
   \begin{tabular}{@{}llrl@{}}
      \toprule
      Policy & Memory & \multicolumn{2}{l}{Return}                                       \\
      \midrule
      \cyclesExactRow{0}{Cycle}{$2$ frames}                                              \\
      \cyclesExactRow{1}{Cycle}{$4$ frames}                                              \\
      \cyclesExactRow{2}{Cycle}{$1$ frame and cause state}                               \\
      \cyclesOptimalRow{3}{Optimal policy}{current state and cause state}                \\
      \midrule
      \cyclesLearnedRow{dqn4}{DQN (4)}{$4$ frames}                                       \\
      \cyclesLearnedRow{tcirl}{TCIRL}{current state and hypothesis state}                \\
      \cyclesLearnedRow{oracle}{Q-learning (cause known)}{current state and cause state} \\
      \bottomrule
   \end{tabular}%
}

\newcommand{\createTransferPlot}[7]{%
   \pgfplotstableread[col sep=comma]{#1}\datatableA
   \pgfplotstableread[col sep=comma]{#2}\datatableB
   \pgfplotstableread[col sep=comma]{#3}\datatableC
   \pgfplotstableread[col sep=comma]{#4}\datatableD
   \begin{tikzpicture}
      \begin{axis}[
            xlabel={Training Steps},
            ylabel={Total Reward per Episode},
            paper plot,
            ymin=#5,
            ymax=#6,
            scaled x ticks=base 10:-3,
            xtick scale label code/.code={},
            xticklabel={\pgfmathprintnumber[fixed,precision=0]{\tick}k},
            #7,
         ]
         \addTrainSeries{tcirl}{\datatableA}{line}{TCIRL}
         \addTrainSeries{exacttransfer}{\datatableB}{line}{TCIRL-Exact}
         \addTrainSeries{notransfer}{\datatableC}{line}{TCIRL-None}
         \addTrainSeries{oracle}{\datatableD}{line}{Q-learning (cause known)}
      \end{axis}
   \end{tikzpicture}%
}
\newcommand{\finalReturnCell}[2]{%
   \pgfplotstablegetelem{0}{#2}\of#1\pgfmathprintnumber[fixed, fixed zerofill, precision=1]{\pgfplotsretval}~%
   \pgfplotstablegetelem{0}{#2_lo}\of#1\edef\finalReturnLo{\pgfplotsretval}%
   \pgfplotstablegetelem{0}{#2_hi}\of#1\edef\finalReturnHi{\pgfplotsretval}%
   {\scriptsize[\pgfmathprintnumber[fixed, fixed zerofill, precision=1]{\finalReturnLo},
         \pgfmathprintnumber[fixed, fixed zerofill, precision=1]{\finalReturnHi}]}%
}
\newcommand{\finalReturnRow}[4]{#4 & \finalReturnCell{#1}{#3} & \finalReturnCell{#2}{#3}}
\newcommand{\createTransferTable}[2]{%
   \pgfplotstableread[col sep=comma]{#1}\transferGene
   \pgfplotstableread[col sep=comma]{#2}\transferTraffic
   \small
   \setlength{\tabcolsep}{3pt}
   \begin{tabular}{@{}lcc@{}}
      \toprule
      variant & genetic therapy & traffic signal                                         \\
      \midrule
      \finalReturnRow{\transferGene}{\transferTraffic}{tcirl}{TCIRL}                     \\
      \finalReturnRow{\transferGene}{\transferTraffic}{exact_transfer}{TCIRL-Exact}      \\
      \finalReturnRow{\transferGene}{\transferTraffic}{no_transfer}{TCIRL-None}          \\
      \finalReturnRow{\transferGene}{\transferTraffic}{oracle}{Q-learning (cause known)} \\
      \bottomrule
   \end{tabular}%
}

\newcommand{\createCounterfactualPlot}[8]{%
   \pgfplotstableread[col sep=comma]{#1}\datatableA
   \pgfplotstableread[col sep=comma]{#2}\datatableB
   \pgfplotstableread[col sep=comma]{#3}\datatableC
   \pgfplotstableread[col sep=comma]{#4}\datatableD
   \begin{tikzpicture}
      \begin{axis}[
            xlabel={Training Steps},
            ylabel={Total Reward per Episode},
            paper plot,
            ymin=#5,
            ymax=#6,
            scaled x ticks=base 10:-3,
            xtick scale label code/.code={},
            xticklabel={\pgfmathprintnumber[fixed,precision=0]{\tick}k},
            #7,
         ]
         \addTrainSeries{tcirl}{\datatableA}{line}{TCIRL}
         \addTrainSeries{counterfactual}{\datatableB}{line}{TCIRL, counterfactual}
         \addTrainSeries{oracle}{\datatableC}{line}{Q-learning (cause known)}
         \addTrainSeries{oraclecounterfactual}{\datatableD}{line}{Q-learning (cause known), counterfactual}
         \if\relax\detokenize{#8}\relax\else
            \pgfplotstableread[col sep=comma]{#8}\optimumtable
            \pgfplotstablegetelem{3}{exact_return}\of\optimumtable
            \xdef\counterfactualReference{\pgfplotsretval}
            \addplot[black, dashed, domain=0:500000, samples=2, forget plot] {\counterfactualReference};
         \fi
      \end{axis}
   \end{tikzpicture}%
}
\newcommand{\createCounterfactualTable}[2]{%
   \pgfplotstableread[col sep=comma]{#1}\counterfactualGene
   \pgfplotstableread[col sep=comma]{#2}\counterfactualTraffic
   \small
   \setlength{\tabcolsep}{3pt}
   \begin{tabular}{@{}lcc@{}}
      \toprule
      variant & genetic therapy & traffic signal                                                                                    \\
      \midrule
      \finalReturnRow{\counterfactualGene}{\counterfactualTraffic}{tcirl}{TCIRL}                                                    \\
      \finalReturnRow{\counterfactualGene}{\counterfactualTraffic}{counterfactual}{TCIRL, counterfactual}                           \\
      \finalReturnRow{\counterfactualGene}{\counterfactualTraffic}{oracle}{Q-learning (cause known)}                                \\
      \finalReturnRow{\counterfactualGene}{\counterfactualTraffic}{oracle_counterfactual}{Q-learning (cause known), counterfactual} \\
      \bottomrule
   \end{tabular}%
}

\newcommand{\stepIndex}{\ensuremath{t}} %
\newcommand{\seqLength}{\ensuremath{n}} %
\newcommand{\lengthBound}{\ensuremath{m}} %
\newcommand{\posIndex}{\ensuremath{i}} %
\newcommand{\posIndexB}{\ensuremath{j}}
\newcommand{\posIndexC}{\ensuremath{k}}
\newcommand{\numStates}{\ensuremath{n}} %
\newcommand{\queueIndex}{\ensuremath{i}}
\newcommand{\seedIndex}{\ensuremath{i}}
\newcommand{\numFrames}{\ensuremath{k}}
\newcommand{\init}{\ensuremath{I}} %
\newcommand{\mathset}[1]{\{ #1 \}}
\newcommand{\setSeparator}{:}
\newcommand{\setBuilder}[2]{\ensuremath{\mathset{#1 \setSeparator #2}}}
\newcommand{\reals}{\mathbb{R}}
\newcommand{\probSimplex}{\ensuremath{\Delta}}
\newcommand{\supp}{\ensuremath{\mathrm{supp}}}

\newcommand{\dfa}[1]{\ensuremath{\mathcal{#1}}}
\newcommand{\dfaInputAlphabet}{\ensuremath{\Sigma}}
\newcommand{\dfaStates}{\ensuremath{Q}}
\newcommand{\dfaCommonState}{\ensuremath{q}}
\newcommand{\dfaInit}{\ensuremath{q_{\init}}}
\newcommand{\dfaTrans}{\ensuremath{\delta}}
\newcommand{\dfaAcc}{\ensuremath{F}}
\newcommand{\languageOf}[1]{\ensuremath{\mathcal{L}(#1)}}
\newcommand{\atomic}{\ensuremath{\mathsf{AP}}}
\newcommand{\propAlphabet}{\ensuremath{2^{\atomic}}}
\newcommand{\labelSequences}{\ensuremath{(\propAlphabet)^{*}}}
\newcommand{\mdpCommonLabel}{\ensuremath{\sigma}} %
\newcommand{\trace}{\ensuremath{w}} %
\newcommand{\labelPrefix}{\ensuremath{u}} %
\newcommand{\labelSuffix}{\ensuremath{v}} %
\newcommand{\atomicProp}{\ensuremath{p}}
\newcommand{\mdp}[1]{\ensuremath{\mathcal{#1}}}
\newcommand{\mdpDiscount}{\ensuremath{\gamma}}
\newcommand{\mdpLabel}{\ensuremath{L}}
\newcommand{\mdpStates}{\ensuremath{S}}
\newcommand{\mdpCommonState}{\ensuremath{s}}
\newcommand{\mdpActions}{\ensuremath{A}}
\newcommand{\mdpCommonAction}{\ensuremath{a}}
\newcommand{\mdpTrans}{\ensuremath{P}}
\newcommand{\mdpReward}{\ensuremath{R}}
\newcommand{\mdpCommonReward}{\ensuremath{r}}
\newcommand{\mdpHistory}{\ensuremath{\xi}} %
\newcommand{\histories}{\ensuremath{H}} %
\newcommand{\policy}{\ensuremath{\pi}}
\newcommand{\qTable}{\ensuremath{\hat{Q}}}
\newcounter{panel}
\renewcommand{\thepanel}{\alph{panel}}
\makeatletter
\@addtoreset{panel}{figure}
\makeatother
\newcommand{\panelLabel}[1]{\refstepcounter{panel}\label{#1}{\footnotesize (\thepanel)}}
\newcommand{\figpanel}[2]{Fig.~\ref{#1}(\ref{#2})}
\newcommand{\panelTitle}[2]{{\footnotesize\mbox{\panelLabel{#1} #2}}}
\newcommand{\initDist}{\ensuremath{\mu}}
\newcommand{\phaseKernel}[1]{\ensuremath{P_{#1}}}
\newcommand{\phaseOne}{\ensuremath{\phaseKernel{1}}}
\newcommand{\phaseTwo}{\ensuremath{\phaseKernel{2}}}
\newcommand{\phaseOf}{\ensuremath{\mathrm{phase}}}
\newcommand{\productMDP}[1]{\ensuremath{\mdp{M}_{#1}}}
\newcommand{\ltlf}{\ensuremath{\text{LTL}_f}}
\newcommand{\nextOp}{\ensuremath{\bigcirc}}
\newcommand{\finallyOp}{\ensuremath{\lozenge}}
\newcommand{\globallyOp}{\ensuremath{\square}}
\newcommand{\untilOp}{\ensuremath{\mathbin{\mathcal{U}}}}
\newcommand{\ltlFormula}{\ensuremath{\varphi}}
\newcommand{\learningRate}{\ensuremath{\alpha}}
\newcommand{\learningRateScale}{\ensuremath{\tau}}
\newcommand{\updateCount}{\ensuremath{\nu}}
\newcommand{\explorationRate}{\ensuremath{\varepsilon}}
\newcommand{\explorationFloor}{\ensuremath{\explorationRate_{\min}}} %
\newcommand{\qInit}{\ensuremath{\hat{Q}_{\mathrm{init}}}}
\newcommand{\trainSteps}{\ensuremath{T}}
\newcommand{\emptyWord}{\ensuremath{\epsilon}}
\newcommand{\witnessProb}{\ensuremath{p_{\trace}}}
\newcommand{\geneSite}[1]{\ensuremath{G_{#1}}}
\newcommand{\medicineCell}{\ensuremath{M}}
\newcommand{\symptomLevel}{\ensuremath{D}}
\newcommand{\harmLevel}{\ensuremath{H}}
\newcommand{\effectInc}{\ensuremath{\mathrm{inc}}}
\newcommand{\effectRecovery}{\ensuremath{\mathrm{dec\_recovery}}}
\newcommand{\queue}[1]{\ensuremath{y_{#1}}}
\newcommand{\actHold}{\ensuremath{h}}
\newcommand{\actAdvance}{\ensuremath{a}}
\newcommand{\geomDist}[1]{\ensuremath{\mathrm{Geom}(#1)}}
\newcommand{\geomParam}{\ensuremath{q}}
\newcommand{\bernoulliDist}[1]{\ensuremath{\mathrm{Bernoulli}(#1)}}
\newcommand{\causeFormula}{\ensuremath{\varphi_c}}
\newcommand{\causeDFA}{\ensuremath{\dfa{A}_c}}
\newcommand{\hypothesisDFA}{\ensuremath{\dfa{H}}}
\newcommand{\finalHypothesis}{\ensuremath{\hypothesisDFA_{\mathrm{final}}}}
\newcommand{\causeStates}{\ensuremath{\dfaStates_c}}
\newcommand{\causeInit}{\ensuremath{\dfaCommonState_{\init}^{c}}}
\newcommand{\causeTrans}{\ensuremath{\dfaTrans_c}}
\newcommand{\causeAcc}{\ensuremath{\dfaAcc_c}}
\newcommand{\causePhase}{\ensuremath{\phaseOf_c}}
\newcommand{\causeStateAt}[1]{\ensuremath{\dfaCommonState^{c}_{#1}}}
\newcommand{\causeSize}{\ensuremath{\numStates_{\causeFormula}}}
\newcommand{\maxCexLength}{\ensuremath{m^*}}

\newcommand{\markerR}{\ensuremath{x_r}}
\newcommand{\markerA}{\ensuremath{x_a}}

\newcommand{\causalSample}{\ensuremath{\mathcal{X}}}
\newcommand{\prefixes}{\ensuremath{\mathrm{Pref}}}
\newcommand{\negativePrefixes}{\ensuremath{N}}
\newcommand{\positivePrefixes}{\ensuremath{P}}
\newcommand{\satFormula}[1]{\ensuremath{\Psi_{#1}^{\causalSample}}}

\newcommand{\satTransVar}[3]{\ensuremath{d_{#1,#2,#3}}}
\newcommand{\satRunVar}[2]{\ensuremath{z_{#1,#2}}}
\newcommand{\satAccVar}[1]{\ensuremath{f_{#1}}}
\spnewtheorem{assumption}{Assumption}{\bfseries}{\itshape}

\begin{document}

\title{
   Temporal-Causal Inference for Reinforcement Learning via Automata Learning%
   \thanks{This is an extended version of the paper accepted for publication in the Proceedings of the 65th IEEE Conference on Decision and Control (CDC), Honolulu, Hawaii, USA, December 2026.}
}
\titlerunning{Temporal-Causal Inference for RL via Automata Learning}
\author{Jan Corazza\inst{1,2} \and Daniil Kaminskyi\inst{1,2} \and Simon Lutz\inst{1,2} \and Patrick Nossol\inst{1,2} \and Hadi~Partovi~Aria\inst{3} \and Zhe Xu\inst{3} \and Daniel Neider\inst{1,2}}
\authorrunning{J. Corazza et al.}
\institute{Research Center Trustworthy Data Science and Security, University Alliance Ruhr, 44227 Dortmund, Germany
   \and TU Dortmund University, 44227 Dortmund, Germany\\
   \email{\{jan.corazza, daniil.kaminskyi, simon.lutz, patrick.nossol, daniel.neider\}@tu-dortmund.de}
   \and School for Engineering of Matter, Transport, and Energy at Arizona State University, Tempe, AZ 85281, USA\\
   \email{\{hpartovi, xzhe1\}@asu.edu}}

\maketitle

\begin{abstract}
   We consider reinforcement learning in environments with dynamics that undergo an irreversible phase transition governed by a hidden temporal pattern.
The agent observes the base state but cannot observe the phase directly.
We formalize this problem as a two-phase non-Markovian decision process and introduce Temporal-Causal Inference for Reinforcement Learning (TCIRL), a framework that jointly learns a control policy and infers the hidden temporal cause of the phase transition.
TCIRL maintains a hypothesis deterministic finite automaton (DFA) to track what phase is active and refines it via counterexample-driven SAT-based synthesis.
We prove that the hypothesis converges almost surely to a DFA whose language induces the same dynamics as the cause language, yielding an optimal policy for the original non-Markovian decision process.
Experiments on a genetic therapy gridworld and a traffic signal environment show that TCIRL recovers the correct cause DFA and matches the full-information baseline in both domains.
\end{abstract}

\section{Introduction}
\label{sec:introduction}

Reinforcement learning (RL) methods for control typically assume Markovian, fully observable dynamics~\cite{sutton_barto2018}.
In many systems, however, the transition dynamics depend on hidden temporal structure that the agent cannot observe directly.
Such history-dependent dynamics give rise to non-Markovian decision processes (NMDPs), where the effect of an action depends on more than the current state~\cite{bacchus_nmdp1997,brafman_rdp2019}.
Two visits to the same state may therefore require different decisions because the governing dynamics depend on the history that led there.
Standard Markovian RL aliases these situations together, so a memoryless policy over the observed state need not be optimal in general~\cite{littman_memoryless1994}.
The goal is to recover a compact finite memory of the past that restores Markovian dynamics in an augmented state space.
Our running example is a medical treatment scenario in which a patient's symptoms respond stochastically to treatment, but the treatment efficacy depends on whether a specific sequence of genetic modifications has been performed, a condition that is not directly observable from symptoms alone.
A standard RL agent learns to treat the symptoms, while our method learns to treat their hidden cause.

We study environments in which a hidden temporal cause governs the transition dynamics (Assumption~\ref{asm:two_phase}).
A hidden temporal logic formula $\causeFormula$ specifies this cause: once the execution so far satisfies $\causeFormula$, the dynamics permanently switch from Phase~1 to Phase~2.
The agent does not observe the phase directly.
Instead, it receives stochastic signals correlated with the active phase, which provide only indirect evidence that the transition has occurred.

\begin{figure}[tbp]
   \begin{minipage}[b]{0.49\textwidth}
      \centering
      \begin{tikzpicture}[every node/.style={font=\scriptsize}, baseline={(0,1.25)}]
         \def\s{0.5}
         \draw (0,0) rectangle (5*\s, 5*\s);
         \foreach \i in {1,...,4} {
               \draw (\i*\s, 0) -- (\i*\s, 5*\s);
               \draw (0, \i*\s) -- (5*\s, \i*\s);
            }
         \node at (1.5*\s, 3.5*\s) {$\geneSite{1}$};
         \node at (3.5*\s, 3.5*\s) {$\geneSite{2}$};
         \node at (1.5*\s, 1.5*\s) {$\geneSite{4}$};
         \node at (3.5*\s, 1.5*\s) {$\geneSite{3}$};
         \node at (2.5*\s, 0.5*\s) {$\medicineCell$};
         \node[draw, inner sep=4pt, align=left, anchor=east]
         (box) at (-0.55, 2.5*\s) {$\symptomLevel \in \{0,\ldots,9\}$\\[2pt]$\harmLevel \in \{0,\ldots,9\}$};
         \node[font=\normalsize] at ($(box.east)!0.5!(0,2.5*\s)$) {$\times$};
      \end{tikzpicture}
   \end{minipage}%
   \hfill
   \begin{minipage}[b]{0.49\textwidth}
      \centering
      \createIntroSymptomPlot[baseline=(current axis.center)]
      {results/gene_editing/q_learning-cause_unknown/trajectories_disease.csv}
      {results/gene_editing/tcirl-sat/trajectories_disease.csv}
   \end{minipage}

   \begin{minipage}{0.49\textwidth}
      \centering
      \panelTitle{fig:gridworld}{Genetic therapy gridworld}
   \end{minipage}%
   \hfill
   \begin{minipage}{0.49\textwidth}
      \centering
      \panelTitle{fig:intro_symptom}{Symptom trajectories}
   \end{minipage}
   \caption{(\ref{fig:gridworld}) The agent navigates a $5 \times 5$ grid with four gene sites $\geneSite{1}$ to $\geneSite{4}$ and a medicine cell $\medicineCell$.
      Site $\geneSite{3}$ is a distractor and does not affect the cause.
      The state includes position, symptom level $\symptomLevel$, and harm level $\harmLevel$.
      (\ref{fig:intro_symptom}) Symptom level within the episodes of the final evaluation checkpoint, after training.
      The TCIRL agent triggers the phase transition around step~$10$, when roughly $52\%$ of its episodes are in Phase~2, after which the symptoms decline.
      Standard Q-learning mostly remains in Phase~1.
      (Mean with $95\%$ bootstrap CI over $50$ seeds.
      The share of episodes is a mean over the same seeds, and both curves average the episodes still running at each step.)}
   \label{fig:intro}
\end{figure}

We illustrate this with a genetic therapy gridworld (\figpanel{fig:intro}{fig:gridworld}).
The agent navigates a $5 \times 5$ grid with four gene sites ($\geneSite{1}$ to $\geneSite{4}$) and a medicine cell ($\medicineCell$).
The state includes a symptom level $\symptomLevel \in \{0,\ldots,9\}$ and a harm level $\harmLevel \in \{0,\ldots,9\}$.
In Phase~1, the symptom level increases by $1$ with probability $0.75$, whereas in Phase~2, it decreases by $1$ with probability $0.75$.
Otherwise, it remains unchanged.
Administering the medicine reduces the symptom level by $5$, but it also resets harm to its maximum, which models the harmful side effects of the medicine.
On every non-medicine step, harm decreases by $1$ until it reaches $0$.
Applying medicine while the current harm level is positive terminates the episode.
The phase transition is triggered by visiting gene sites $\geneSite{1}, \geneSite{2}, \geneSite{1}, \geneSite{4}$  in that order, with other visits allowed between them.
$\geneSite{3}$ is a distractor site which does not appear in the cause formula.

\figpanel{fig:intro}{fig:intro_symptom} shows that our method, Temporal-Causal Inference for Reinforcement Learning (TCIRL), learns the correct gene-site sequence to trigger Phase~2, after which the symptoms decline steadily.
Standard Q-learning~\cite{watkins_dayan1992}, which lacks memory for the required temporal pattern, remains in Phase~1 in $58\%$ of the evaluation episodes.

Our approach maintains a hypothesis deterministic finite automaton (DFA) $\hypothesisDFA$ that models the temporal cause.
Starting from a trivial $1$-state hypothesis which rejects every label sequence (the sequence of high-level events of an episode, such as gene-site visits), the agent refines $\hypothesisDFA$ over training by exploiting the stochastic evidence available in each episode.
The method requires domain knowledge that specifies exclusive characteristics of each phase.
We call these \emph{markers}: in the genetic therapy example, a symptom increase can only occur in Phase~1, and a spontaneous recovery can only occur in Phase~2.
After each episode, the agent checks whether the episode label sequence is inconsistent with its current hypothesis.
If so, a Boolean satisfiability (SAT) solver synthesizes a new minimal DFA consistent with every stored counterexample.

Markers occur probabilistically and need not appear immediately before or after a phase transition.
An episode may contain no marker, and observations without a marker may be compatible with either phase.
JIRP (joint inference of reward machines and policies)~\cite{jirp}, in contrast, observes at every step both the label and the reward output of the reward machine it infers, disambiguating it in each step.
However, we assume that the markers are phase-exclusive and that, for every wrong hypothesis, there exists some reachable episode whose markers expose the error (Assumption~\ref{asm:markers}).
Under this assumption, we prove that the hypothesis converges almost surely to a DFA whose language induces the same dynamics as the cause language (Theorem~\ref{thm:exact_recovery}), and the agent converges to an optimal policy (Theorem~\ref{thm:convergence}).

We validate TCIRL on two domains with opposite causal effects, and compare its performance to a baseline with full access to the causal structure from the start.
In genetic therapy, triggering the cause is beneficial (symptoms decline).
In a traffic signal environment, triggering it is harmful (congestion increases).
In both cases, TCIRL recovers the correct cause DFA and matches the full-information baseline.

Our work makes the following contributions:
\begin{enumerate}
   \item A two-phase NMDP formalization for environments with hidden temporal causes, and TCIRL, an algorithm that jointly learns a policy and infers a cause DFA via counterexample-driven SAT-based synthesis.
   \item Almost-sure convergence guarantees: the inferred hypothesis converges to a DFA whose language induces the same dynamics as the cause language (Theorem~\ref{thm:exact_recovery}), and the policy converges to the optimal value of the NMDP (Theorem~\ref{thm:convergence}).
   \item Experimental validation on genetic therapy and traffic signal control, showing that TCIRL recovers the correct cause DFA and matches the full-information baseline.
\end{enumerate}
The code of TCIRL and of every experiment is published at \url{https://github.com/verified-ai/tcirl}, release~v1.1.1.

\section{Related Work}
\label{sec:related_work}

TCIRL learns a control policy while inferring the temporal structure behind a change in dynamics.
We first place our setting among models of history-dependent dynamics, then review methods that assume that the temporal structure is known beforehand versus those which  infer it during RL, and finally relate our use of the term causal to causal RL.

\subsection{History-Dependent Dynamics}
\label{sec:related_models}

Non-Markovian decision processes (NMDPs) generalize Markov decision processes (MDPs) by allowing rewards and dynamics to depend on the history.
Brafman and De Giacomo~\cite{brafman_rdp2019} formalize regular decision processes (RDPs), where both rewards and dynamics may depend on regular properties of the history.
Our setting is a structured special case: a single temporal condition over labels induces a permanent switch between two Markovian transition kernels, while the reward remains Markovian.

Brafman and De Giacomo also show that RDPs correspond to partially observable MDPs (POMDPs) whose hidden variables are regular functions of the observable history.
The phase in our setting is such a variable, so a DFA state suffices to track it, rather than the belief state that POMDP solution methods maintain under a known model~\cite{smallwood_sondik1973,kaelbling_pomdp1998}.
Hidden-mode MDPs~\cite{choi_hidden_mode2001} and hidden-parameter MDPs~\cite{doshi_velez_hip_mdp2016} also let a latent variable select the Markovian dynamics.
In hidden-mode MDPs, the mode follows a Markov chain that the agent's actions do not influence, and hidden-parameter MDPs fix the latent parameter for the duration of each task.
In our setting, the agent's own label history triggers the switch, so the agent can learn to cause or avoid it.

\subsection{Known Temporal Structure}
\label{sec:related_known}

When the temporal structure is known, auxiliary state variables or an automaton restore the Markov property.
Bacchus et al.~\cite{bacchus_nmdp1997} compile temporal logic specifications of history-dependent rewards and dynamics into such variables, and Brafman and De Giacomo~\cite{brafman_rdp2019} solve an RDP by compiling its formulas into automata.
In planning, RL and control, temporal-logic specifications are likewise compiled into product MDPs, either as a control objective or as a reward description~\cite{sadigh_specifications2014,hasanbeig_ltl_rl2019,bacchus_non_markov1996,camacho_non_markov2017}.
The cause in our setting is also a temporal logic formula, written in Linear Temporal Logic over finite traces ($\ltlf$), and every $\ltlf$ formula compiles into a DFA~\cite{degiacomo_vardi2013}.
Unlike the specifications in these works, however, the cause formula is hidden from the agent, so TCIRL must infer a DFA from data instead of compiling one.

Reward machines~\cite{reward_machines_icarte} expose non-Markovian reward structure as a finite automaton and support counterfactual updates, which use each observed transition to update the Q-values of all automaton states.
They assume Markovian transition dynamics and a history-dependent reward, whereas we assume history-dependent dynamics and a Markovian reward, so our automaton governs what an action does rather than its reward.
Our full-information baseline, Q-learning (cause known), is the counterpart of Q-learning with a known reward machine: it learns on the product of the NMDP with the cause DFA.
Counterfactual updates also apply to TCIRL and to this baseline, but an observed transition can update the Q-values of only those automaton states that select the same transition kernel as the agent's current automaton state (Section~\ref{sec:counterfactual}).

\subsection{Learning Automata During RL}
\label{sec:related_learning}

Several methods alternate RL with inferring an automaton from the experience it generates, and they differ in what the automaton represents and where the supervision for automaton learning comes from.
Abadi and Brafman~\cite{abadi_rdp2020} address the closest problem, learning history-dependent dynamics, and JIRP~\cite{jirp} uses the closest mechanism, a SAT-based counterexample loop around model-free RL.

Brafman and De Giacomo only sketch how an RDP could be learned from traces, and Abadi and Brafman give the first algorithm that does so.
Their model-based method clusters sampled histories by empirical next-state distributions, learns a Mealy machine that maps each history to its cluster, and plans with Monte Carlo tree search using the estimated per-cluster dynamics.
With one of its sampling strategies, it runs Q-learning on the product of the observed state and the current Mealy machine, as TCIRL does with its hypothesis DFA.
The key difference is how each method decides that two histories behave differently.
Abadi and Brafman compare estimated next-state distributions, which requires repeated visits to similar histories and separates them only approximately.
A marker is a proposition that occurs with probability zero in one phase, so a single occurrence certifies the phase of the preceding prefix, and TCIRL imposes this certificate as a hard constraint on the SAT-based synthesis of a minimal DFA.
The markers require domain knowledge that Abadi and Brafman do not assume, and in exchange they enable almost-sure recovery of the cause and convergence to an optimal policy (Theorems~\ref{thm:exact_recovery} and~\ref{thm:convergence}), which their method does not guarantee.

JIRP jointly infers reward machines and policies in a SAT-based counterexample loop, and Xu et al.~\cite{xu_active_fra2021} infer the reward machine by active learning with queries and counterexamples instead, which reaches an optimal policy faster.
TCIRL follows the SAT encoding of JIRP, but its automaton governs transition dynamics rather than rewards.
As a result, counterexamples come from the stochastic phase-exclusive markers, not from reward traces.
The reward machine inferred by JIRP, even if that algorithm converged in our setting, could not monitor the cause, because rewards do not track the phase transition directly.

At the automata-learning level, Heule and Verwer~\cite{heule_verwer2010} identify DFAs exactly from accepted and rejected words.
Here the supervision is prefix-level but partial: each sampled trace certifies an initial block of rejecting prefixes and a final block of accepting prefixes, leaving at most an uncertified band between them.
The SAT step therefore receives richer prefix-level constraints than standard passive DFA learning, though less complete information than JIRP, which observes an output of the reward machine at every prefix.

\subsection{Causal Reasoning in RL}
\label{sec:related_causal}

Causal RL aims to exploit causal structure for more sample-efficient learning~\cite{zeng_causal_rl_survey2025,lu_causal_rl2018}.
Our use of the term causal is narrower: we refer to a temporal pattern that changes the transition dynamics, not to an interventionist causal model in the sense of Pearl~\cite{pearl_causality2009}.
The closest work in this area is STL-CIRL~\cite{pmlr-v288-partovi-aria25a}, which uses notions from Actual Causality~\cite{actual_causality} to mine causal Signal Temporal Logic specifications via counterexample-guided refinement and directly inspired our benchmark environments, even though we substantially changed their semantics.
TCIRL differs by using $\ltlf$ over finite traces with DFA-based monitoring, and by focusing on hidden transition dynamics rather than reward structure.
In the environments of STL-CIRL, the state records the variables over which the causes are expressed, such as the mutation status of each gene, while in our setting the cause is a temporal pattern that the state does not record, which makes the dynamics non-Markovian.
We also provide almost-sure convergence guarantees under the marker assumption.

\section{Preliminaries}
\label{sec:background}

We study a restricted class of NMDPs in which the non-Markovian dependence arises from a hidden causal mechanism.
Moreover, we assume that transitions in the NMDP are labeled with sets of atomic propositions, which denote high-level events that occur in the transition.
The system evolves in one of two Markovian phases, and the active phase is determined by whether a temporal pattern in the label history has occurred.
This setting is common in control, where unobserved mode switches alter the local dynamics while the controller still observes only the physical state.
To specify such patterns, we use Linear Temporal Logic over finite traces ($\ltlf$)~\cite{degiacomo_vardi2013}, which can express sequencing, eventuality, and persistence constraints over propositional labels.
The key property we exploit is that every $\ltlf$ formula can be compiled into a DFA, so monitoring whether the pattern has occurred reduces to tracking a finite automaton state.

\begin{definition}[Labeled NMDP]
   \label{def:nmdp}
   A labeled non-Markovian decision process (NMDP) is a tuple $\mdp{N} = (\mdpStates, \mdpActions, \mdpTrans, \mdpReward, \initDist, \mdpDiscount, \atomic, \mdpLabel)$ where $\mdpStates$ is a finite state space, $\mdpActions$ is a finite action space, $\mdpTrans \colon (\mdpStates \times \mdpActions)^+ \to \probSimplex(\mdpStates)$ is a history-dependent transition function, $\mdpReward \colon \mdpStates \times \mdpActions \times \mdpStates \to \reals$ is a Markovian reward function, $\initDist$ is an initial state distribution, $\mdpDiscount \in [0,1)$ is the discount factor, $\atomic$ is a finite set of atomic propositions, and $\mdpLabel \colon \mdpStates \times \mdpActions \times \mdpStates \to \propAlphabet$ labels each transition with a set of atomic propositions.
   Here, $\probSimplex(\mdpStates)$ is the set of probability distributions over $\mdpStates$.
\end{definition}

Let $\histories = \mdpStates \times (\mdpActions \times \mdpStates)^*$ be the set of histories.
A history $\mdpHistory_{\seqLength} = \mdpCommonState_0, \mdpCommonAction_0, \ldots, \mdpCommonState_{\seqLength} \in \histories$ has $\seqLength$ steps, and for $0 \le \stepIndex \le \seqLength$, $\mdpHistory_{\stepIndex}$ denotes its prefix up to $\mdpCommonState_{\stepIndex}$.
An agent that interacts with the NMDP generates a trajectory $\mdpCommonState_0, \mdpCommonAction_0, \mdpCommonState_1, \mdpCommonAction_1, \ldots$ step by step: the first state is $\mdpCommonState_0 \sim \initDist$, and after the history $\mdpHistory_{\stepIndex} = \mdpCommonState_0, \mdpCommonAction_0, \ldots, \mdpCommonState_{\stepIndex}$ and the agent's action $\mdpCommonAction_{\stepIndex}$, the next state is $\mdpCommonState_{\stepIndex+1} \sim \mdpTrans(\,\cdot \mid \mdpHistory_{\stepIndex}, \mdpCommonAction_{\stepIndex})$.
A history $\mdpHistory_{\seqLength}$ is \emph{attainable} if $\mdpCommonState_0 \in \supp(\initDist)$, the support of $\initDist$, and $\mdpTrans(\mdpCommonState_{\stepIndex+1} \mid \mdpHistory_{\stepIndex}, \mdpCommonAction_{\stepIndex}) > 0$ for every $\stepIndex < \seqLength$, and \emph{$\lengthBound$-attainable} if moreover $\seqLength \le \lengthBound$.
Each transition produces a label $\mdpLabel(\mdpCommonState_{\stepIndex}, \mdpCommonAction_{\stepIndex}, \mdpCommonState_{\stepIndex+1}) \in \propAlphabet$, and every history a label sequence, defined by $\mdpLabel(\mdpHistory_0) = \emptyWord$, the empty label sequence, and $\mdpLabel(\mdpHistory_{\seqLength}) = \mdpLabel(\mdpCommonState_0, \mdpCommonAction_0, \mdpCommonState_1) \cdots \mdpLabel(\mdpCommonState_{\seqLength-1}, \mdpCommonAction_{\seqLength-1}, \mdpCommonState_{\seqLength})$.
For a label sequence $\trace = \mdpCommonLabel_0 \cdots \mdpCommonLabel_{\seqLength-1} \in \labelSequences$, we write $|\trace| = \seqLength$ for its length.
A label sequence $\labelPrefix$ is a prefix of $\trace$, and $\trace$ an extension of $\labelPrefix$, if $\trace = \labelPrefix \labelSuffix$ for some (possibly empty) $\labelSuffix$.
A label sequence is ($\lengthBound$-)attainable if an ($\lengthBound$-)attainable history produces it.

Our focus is a structured hidden-state problem in which the relevant past is summarized by a single regular temporal condition on the label history.

\begin{definition}[$\ltlf$ Syntax and Semantics]
   \label{def:ltlf}
   $\ltlf$ (Linear Temporal Logic over finite traces) formulas over $\atomic$ are defined by the grammar $\ltlFormula ::= \atomicProp \mid \neg \ltlFormula \mid \ltlFormula_1 \!\land\! \ltlFormula_2 \mid \nextOp \ltlFormula \mid \ltlFormula_1 \untilOp \ltlFormula_2$ where $\atomicProp \in \atomic$.
   Derived operators include $\finallyOp \ltlFormula \equiv \top \untilOp \ltlFormula$ (eventually) and $\globallyOp \ltlFormula \equiv \neg \finallyOp \neg \ltlFormula$ (always).
   Satisfaction is defined over finite traces $\trace = \mdpCommonLabel_0 \cdots \mdpCommonLabel_{\seqLength-1} \in \labelSequences$, including the empty trace $\emptyWord$ with $\seqLength = 0$, at indices $0 \le \posIndex \le \seqLength$:
   $\trace, \posIndex \models \atomicProp$ iff $\posIndex < \seqLength$ and $\atomicProp \in \mdpCommonLabel_{\posIndex}$;
   $\trace, \posIndex \models \neg \ltlFormula$ iff $\trace, \posIndex \not\models \ltlFormula$;
   $\trace, \posIndex \models \ltlFormula_1 \land \ltlFormula_2$ iff $\trace, \posIndex \models \ltlFormula_1$ and $\trace, \posIndex \models \ltlFormula_2$;
   $\trace, \posIndex \models \nextOp \ltlFormula$ iff $\posIndex < \seqLength{-}1$ and $\trace, \posIndex{+}1 \models \ltlFormula$;
   $\trace, \posIndex \models \ltlFormula_1 \untilOp \ltlFormula_2$ iff there exists $\posIndex \le \posIndexB < \seqLength$ such that $\trace, \posIndexB \models \ltlFormula_2$ and $\trace, \posIndexC \models \ltlFormula_1$ for all $\posIndex \le \posIndexC < \posIndexB$.
   We write $\trace \models \ltlFormula$ for $\trace, 0 \models \ltlFormula$.
   The language of $\ltlFormula$ is $\languageOf{\ltlFormula} = \setBuilder{\trace \in \labelSequences}{\trace \models \ltlFormula}$.
\end{definition}

We now formalize this two-phase causal structure.
The non-Markovian dependence is captured by a single $\ltlf$ formula $\causeFormula$, which we call the cause formula.
We refer to $\languageOf{\causeFormula}$ as the cause language of the NMDP.
When the label history satisfies $\causeFormula$, the system switches permanently from Phase~1 to Phase~2 dynamics.

\begin{assumption}[Two-Phase NMDP]
   \label{asm:two_phase}
   There exist Markovian kernels $\phaseOne, \phaseTwo \colon \allowbreak \mdpStates \times \mdpActions \times \mdpStates \to [0,1]$ and an $\ltlf$ formula $\causeFormula$ over $\atomic$ such that for every history $\mdpHistory_{\seqLength} = \mdpCommonState_0, \mdpCommonAction_0, \ldots, \mdpCommonState_{\seqLength}$, action $\mdpCommonAction \in \mdpActions$ and state $\mdpCommonState_{\seqLength+1} \in \mdpStates$,
   \begin{align}
      \label{eq:two_phase}
      \mdpTrans(\mdpCommonState_{\seqLength+1} \mid \mdpHistory_{\seqLength}, \mdpCommonAction) = \begin{cases}
                                                                                                     \phaseTwo(\mdpCommonState_{\seqLength+1} \mid \mdpCommonState_{\seqLength}, \mdpCommonAction) & \text{if } \mdpLabel(\mdpHistory_{\seqLength}) \models \causeFormula, \\
                                                                                                     \phaseOne(\mdpCommonState_{\seqLength+1} \mid \mdpCommonState_{\seqLength}, \mdpCommonAction) & \text{otherwise.}
                                                                                                  \end{cases}
   \end{align}
   The formula $\causeFormula$ is closed under extension: $\trace \models \causeFormula$ implies $\trace \labelSuffix \models \causeFormula$ for all~$\labelSuffix$.
\end{assumption}

The reward $\mdpReward$ is Markovian, and only the cause formula $\causeFormula$ introduces dependence on the past.
The agent observes $\mdpCommonState_{\stepIndex}$ but not the phase, so the causal mechanism is hidden.
Extension-closure captures the irreversibility of the effect: once the cause fires, it remains satisfied regardless of future observations.

The system can terminate.
A state $\mdpCommonState$ is terminal if $\phaseOne(\mdpCommonState \mid \mdpCommonState, \mdpCommonAction) = \phaseTwo(\mdpCommonState \mid \mdpCommonState, \mdpCommonAction) = 1$, $\mdpReward(\mdpCommonState, \mdpCommonAction, \mdpCommonState) = 0$ and $\mdpLabel(\mdpCommonState, \mdpCommonAction, \mdpCommonState) = \emptyset$ for every action $\mdpCommonAction \in \mdpActions$.
A history $\mdpHistory_{\seqLength}$ has terminated if its last state $\mdpCommonState_{\seqLength}$ is terminal, and is unterminated otherwise.
Both cases of Equation~\eqref{eq:two_phase} agree after termination.

Because $\causeFormula$ is an $\ltlf$ formula, it can be compiled into a DFA that monitors the label history and tracks which phase is active.
This is the finite-memory structure that TCIRL seeks to recover from data.

\begin{definition}[DFA]
   \label{def:dfa}
   A deterministic finite automaton (DFA) is a tuple $\dfa{A} = (\dfaStates, \dfaInit, \dfaInputAlphabet, \dfaTrans, \dfaAcc)$ with finite state set $\dfaStates$, initial state $\dfaInit \in \dfaStates$, alphabet $\dfaInputAlphabet$, transition function $\dfaTrans \colon \dfaStates \times \dfaInputAlphabet \to \dfaStates$, and accepting states $\dfaAcc \subseteq \dfaStates$.
   We write $\dfaTrans^*(\dfaCommonState, \trace)$ for the state reached from $\dfaCommonState$ after processing word $\trace$.
   Its language is $\languageOf{\dfa{A}} = \setBuilder{\trace \in \dfaInputAlphabet^*}{\dfaTrans^*(\dfaInit, \trace) \in \dfaAcc}$.
   A DFA is successor-closed if $\forall \dfaCommonState \in \dfaAcc,\, \forall \mdpCommonLabel \in \dfaInputAlphabet\!: \dfaTrans(\dfaCommonState, \mdpCommonLabel) \in \dfaAcc$.
\end{definition}

By the $\ltlf$-to-DFA correspondence~\cite{degiacomo_vardi2013}, the cause language is recognized by some DFA.
Because $\causeFormula$ is closed under extension, there exists a successor-closed DFA $\causeDFA$ over $\propAlphabet$ with $\languageOf{\causeDFA} = \languageOf{\causeFormula}$.
Once a run enters its accepting set, every continuation remains accepting.
Such a DFA is not unique, and many successor-closed DFAs recognize the same cause language.
We write $\causeSize$ to denote the minimum state count among all of them.

\section{Method}
\label{sec:method}

TCIRL jointly learns a policy and infers the hidden cause DFA from episode data.
Once a correct cause monitor is learned, the problem reduces to a standard product-MDP construction.
The difficulty is to identify that monitor from partial stochastic evidence gathered during RL.
We first demonstrate how a hypothesis DFA induces a product MDP.
Second, we define the stochastic markers that provide indirect evidence about the cause and formalize our method for counterexample detection.
Then we describe the SAT-based synthesis step used to update the hypothesis.
Finally, we summarize the overall training procedure, bound the required episode length, and state our convergence guarantees.

\subsection{Problem Setting}
\label{sec:problem_setting}

Any successor-closed DFA recognizing the cause language $\languageOf{\causeFormula}$ is a phase monitor.
Stepping it on each transition label and checking acceptance determines whether $\causeFormula$ has been satisfied.
However, the agent does not know $\causeFormula$ and therefore cannot construct such a DFA directly.
Instead, it maintains a hypothesis DFA $\hypothesisDFA = (\dfaStates, \dfaInit, \propAlphabet, \dfaTrans, \dfaAcc)$ and interprets a state $\dfaCommonState \in \dfaAcc$ as Phase~2 and a state $\dfaCommonState \notin \dfaAcc$ as Phase~1.
We write $\phaseOf(\dfaCommonState) = 2$ if $\dfaCommonState \in \dfaAcc$ and $1$ otherwise, and $\dfaStates_{\hypothesisDFA}$, $\dfaCommonState^{\hypothesisDFA}_{\init}$, $\dfaTrans_{\hypothesisDFA}$ and $\dfaAcc_{\hypothesisDFA}$ when the components of $\hypothesisDFA$ must be told apart from those of another DFA.
The product MDP $\productMDP{\hypothesisDFA} = (\mdpStates \times \dfaStates, \initDist \times \{\dfaInit\}, \mdpActions, \mdpTrans_{\hypothesisDFA}, \mdpReward_{\hypothesisDFA}, \mdpDiscount)$ starts in $(\mdpCommonState_0, \dfaInit)$ with $\mdpCommonState_0 \sim \initDist$ and has transition kernel
\begin{align*}
   \mdpTrans_{\hypothesisDFA} & ((\mdpCommonState, \dfaCommonState), \mdpCommonAction, (\mdpCommonState', \dfaCommonState'))                                                                                                                                      \\
                              & = \begin{cases}
                                     \phaseKernel{\phaseOf(\dfaCommonState)}(\mdpCommonState' \mid \mdpCommonState, \mdpCommonAction) & \text{if } \dfaCommonState' = \dfaTrans(\dfaCommonState, \mdpLabel(\mdpCommonState, \mdpCommonAction, \mdpCommonState')), \\
                                     0                                                                                                & \text{otherwise,}
                                  \end{cases}
\end{align*}
and Markovian reward $\mdpReward_{\hypothesisDFA}((\mdpCommonState, \dfaCommonState), \mdpCommonAction, (\mdpCommonState', \dfaCommonState')) = \mdpReward(\mdpCommonState, \mdpCommonAction, \mdpCommonState')$.
If $\labelPrefix \in \languageOf{\hypothesisDFA} \iff \labelPrefix \models \causeFormula$ for the label sequence $\labelPrefix$ of every unterminated attainable history, then $\productMDP{\hypothesisDFA}$ has the same transition dynamics as the NMDP on the states that attainable histories reach, because both cases of Equation~\eqref{eq:two_phase} agree after termination.
In Theorem~\ref{thm:convergence}, we show that the optimal values of $\productMDP{\hypothesisDFA}$ (for such $\hypothesisDFA$) and the NMDP coincide.

\subsection{Marker Assumption}
\label{sec:markers}

Each episode generates a label sequence over $\propAlphabet$.
TCIRL uses that same sequence both to run the current hypothesis DFA and, once the episode completes, to read the marker evidence against which the hypothesis is tested.

The agent cannot observe the phase directly.
Instead, TCIRL relies on two atomic propositions $\markerR, \markerA \in \atomic$ whose occurrence in transition labels is exclusive to one phase.
The rejection marker $\markerR$ can only appear in Phase~1, certifying that the cause has not yet fired, and the acceptance marker $\markerA$ can only appear in Phase~2, certifying that it has.
In the genetic therapy example, a symptom increase is $\markerR$ and a spontaneous recovery (as opposed to one caused by directly applying the medicine) is $\markerA$.
Identifying such markers requires domain knowledge about the phase-dependent dynamics, but not knowledge of the cause formula $\causeFormula$ itself.

The phase of a label sequence $\labelPrefix$ acts on the transition that follows $\labelPrefix$, and a marker can reveal it only in the label of such a subsequent transition.

\begin{assumption}[Marker Soundness and Completeness]
   \label{asm:markers}
   Let $\markerR, \markerA \in \atomic$ be the markers that TCIRL receives.
   They have the following properties.
   \begin{enumerate}
      \item \textbf{Soundness.}
            Let $\labelPrefix\mdpCommonLabel$ be an attainable label sequence.
            If $\markerR \in \mdpCommonLabel$, then $\labelPrefix \not\models \causeFormula$.
            If $\markerA \in \mdpCommonLabel$, then $\labelPrefix \models \causeFormula$.
      \item \textbf{Completeness.}
            Let $\labelPrefix = \mdpLabel(\mdpHistory_{\seqLength})$ for an unterminated attainable history $\mdpHistory_{\seqLength}$.
            If $\labelPrefix \not\models \causeFormula$, then $\markerR \in \mdpCommonLabel'$ for some attainable extension $\labelPrefix \labelSuffix\mdpCommonLabel'$.
            If $\labelPrefix \models \causeFormula$, then $\markerA \in \mdpCommonLabel'$ for some prefix $\labelPrefix'$ of $\labelPrefix$ and some attainable label sequence $\labelPrefix'\mdpCommonLabel'$.
   \end{enumerate}
\end{assumption}

Soundness says that a marker in a label certifies the phase of the prefix before that label.
Completeness says that such a certificate exists after some extension of $\labelPrefix$ for $\markerR$ and after some prefix of $\labelPrefix$ for $\markerA$.
The asymmetry comes from the hypothesis, which is successor-closed.
A hypothesis that wrongly accepts $\labelPrefix$ also accepts every extension of $\labelPrefix$, so an $\markerR$ after any extension of $\labelPrefix$ exposes the error.
A hypothesis that wrongly rejects $\labelPrefix$ also rejects every prefix of $\labelPrefix$, so an $\markerA$ after any prefix of $\labelPrefix$ exposes the error.
Prefix and extension include $\labelPrefix$ itself.
Completeness asks nothing after termination, where every label is $\emptyset$ and no marker can occur.

The markers may be probabilistic: completeness requires only that such a label sequence is attainable, not that it occurs in a given episode.
Assumption~\ref{asm:markers} only guarantees that every wrong hypothesis has an attainable counterexample (Section~\ref{sec:mstar}).

\subsection{Counterexample Detection}
\label{sec:counterexample}

A completed episode yields an attainable label sequence $\trace$.
We say $\trace$ is a counterexample to hypothesis $\hypothesisDFA$ if some prefix $\labelPrefix\mdpCommonLabel$ of $\trace$ has either $\markerR \in \mdpCommonLabel$ and $\hypothesisDFA$ accepting $\labelPrefix$ (false positive), or $\markerA \in \mdpCommonLabel$ and $\hypothesisDFA$ rejecting $\labelPrefix$ (false negative).
By soundness, the marker in $\mdpCommonLabel$ gives the true phase of $\labelPrefix$, so a counterexample proves that $\hypothesisDFA$ misclassifies $\labelPrefix$.
Detection runs $\hypothesisDFA$ on the label sequence and compares each hypothesis state with the marker in the next label, where it exists (Algorithm~\ref{alg:tcirl}, line~\ref{alg:line:isce}).
TCIRL retains detected counterexamples in a set $\causalSample$, which we call the sample (line~\ref{alg:line:store}).

\subsection{SAT-Based Cause DFA Synthesis}
\label{sec:sat_synthesis}

When a counterexample is detected, TCIRL re-synthesizes the hypothesis from scratch as a minimal DFA consistent with the stored counterexamples (line~\ref{alg:line:sat}).
We call this step a re-inference.
Let $\prefixes(\causalSample)$ denote the set of prefixes of the observed counterexamples.
From it, we extract the finite certified prefix sets
\begin{align*}
   \negativePrefixes(\causalSample) & = \setBuilder{\labelPrefix}{\labelPrefix\mdpCommonLabel \in \prefixes(\causalSample) \text{ for some } \mdpCommonLabel \ni \markerR}, \\
   \positivePrefixes(\causalSample) & = \setBuilder{\labelPrefix}{\labelPrefix\mdpCommonLabel \in \prefixes(\causalSample) \text{ for some } \mdpCommonLabel \ni \markerA}.
\end{align*}
A DFA is consistent with $\causalSample$ if it rejects every $\labelPrefix \in \negativePrefixes(\causalSample)$ and accepts every $\labelPrefix \in \positivePrefixes(\causalSample)$.
By soundness, every $\labelPrefix \in \negativePrefixes(\causalSample)$ satisfies $\labelPrefix \not\models \causeFormula$ and every $\labelPrefix \in \positivePrefixes(\causalSample)$ satisfies $\labelPrefix \models \causeFormula$, so every DFA recognizing $\languageOf{\causeFormula}$ is consistent with every sample.

Following JIRP~\cite{jirp}, we encode the existence of a consistent DFA with a fixed state count $\numStates$ and states $\{0, \ldots, \numStates{-}1\}$ as a propositional formula $\satFormula{\numStates}$ over three sets of variables: transition variables $\satTransVar{\dfaCommonState}{\mdpCommonLabel}{\dfaCommonState'}$ asserting that state $\dfaCommonState$ moves to $\dfaCommonState'$ on symbol $\mdpCommonLabel$, acceptance variables $\satAccVar{\dfaCommonState}$ marking state $\dfaCommonState$ as accepting, and run variables $\satRunVar{\labelPrefix}{\dfaCommonState}$ asserting that the DFA run on prefix $\labelPrefix \in \prefixes(\causalSample)$ ends in state~$\dfaCommonState$.
The states $\dfaCommonState$ and $\dfaCommonState'$ range over $\{0, \ldots, \numStates{-}1\}$, and the symbol $\mdpCommonLabel$ over $\propAlphabet$.
The run variables range over prefixes of the observed counterexamples, while $\negativePrefixes(\causalSample)$ and $\positivePrefixes(\causalSample)$ constrain whether the reached states must be rejecting or accepting.
The formula is the conjunction of the following constraints:
\begin{enumerate}
   \item\label{itm:sat_determinism_totality} \emph{Determinism and totality:} for each $(\dfaCommonState, \mdpCommonLabel)$, exactly one $\dfaCommonState'$ satisfies $\satTransVar{\dfaCommonState}{\mdpCommonLabel}{\dfaCommonState'}$.
   \item\label{itm:sat_execution_consistency} \emph{Execution base and consistency:} $\satRunVar{\emptyWord}{0}$ holds, exactly one state is assigned to each prefix in $\prefixes(\causalSample)$, and for each non-empty prefix $\labelPrefix\mdpCommonLabel$: $\satRunVar{\labelPrefix}{\dfaCommonState} \!\land\! \satTransVar{\dfaCommonState}{\mdpCommonLabel}{\dfaCommonState'} \Rightarrow \satRunVar{\labelPrefix\mdpCommonLabel}{\dfaCommonState'}$.
   \item\label{itm:sat_successor_closed_acceptance} \emph{Successor-closed acceptance:} $\satAccVar{\dfaCommonState} \land \satTransVar{\dfaCommonState}{\mdpCommonLabel}{\dfaCommonState'} \Rightarrow \satAccVar{\dfaCommonState'}$ for all $\mdpCommonLabel$.
   \item\label{itm:sat_consistency_constraints} \emph{Consistency constraints:} $\satRunVar{\labelPrefix}{\dfaCommonState} \Rightarrow \neg\,\satAccVar{\dfaCommonState}$ for $\labelPrefix \in \negativePrefixes(\causalSample)$, and $\satRunVar{\labelPrefix}{\dfaCommonState} \Rightarrow \satAccVar{\dfaCommonState}$ for $\labelPrefix \in \positivePrefixes(\causalSample)$.
\end{enumerate}
Constraint~\ref{itm:sat_successor_closed_acceptance} reflects extension-closure of $\causeFormula$.
A DFA that satisfies it rejects every prefix of a rejected word and accepts every extension of an accepted word.
The formula $\satFormula{\numStates}$ has two properties: it is satisfiable if and only if there exists a consistent successor-closed $\numStates$-state DFA, and any satisfying assignment contains enough information to reconstruct such a DFA, namely the transition table from $\satTransVar{\dfaCommonState}{\mdpCommonLabel}{\dfaCommonState'}$ and the accepting states from $\satAccVar{\dfaCommonState}$.
TCIRL solves $\satFormula{\numStates}$ for $\numStates = 1, 2, 3, \ldots$ and takes the DFA of the first satisfiable instance as the new hypothesis $\hypothesisDFA$, which is therefore minimal.
A successor-closed DFA recognizing $\languageOf{\causeFormula}$ exists (Section~\ref{sec:background}) and is consistent with every sample, so $\satFormula{\numStates}$ is satisfiable for its state count and the search terminates.

\subsection{Training Procedure}
\label{sec:training_procedure}

Algorithm~\ref{alg:tcirl} contains the pseudocode for our method.
The agent performs $\explorationRate$-greedy Q-learning on the product of the base state and the current hypothesis state.
At each step, $\explorationRate$ is set by the exploration schedule, which continues across re-inferences.
Label sequences from completed episodes are added to $\causalSample$ only when they are counterexamples to the current hypothesis.
After each observed counterexample, our SAT-based synthesis method described in Section~\ref{sec:sat_synthesis} recomputes the smallest consistent DFA.
The new hypothesis always differs from the old one, which misclassifies the new counterexample, and its states index the Q-table from then on.
TCIRL starts each of them from the Q-values of a state of the old hypothesis chosen by a similarity heuristic (Section~\ref{sec:value_transfer}).
An episode ends when it enters a terminal state or reaches the episode limit, the largest number of steps an episode may take.
At a terminal state $\mdpCommonState_{\stepIndex+1}$, the maximum in the Q-update of Algorithm~\ref{alg:tcirl} is $0$, the value of a terminal state.

\begin{algorithm}[tbp]
   \caption{TCIRL: Temporal-Causal Inference for RL}
   \label{alg:tcirl}
   \begin{algorithmic}[1]
      \State \textbf{Input:} NMDP $\mdp{N}$, markers $\markerR, \markerA$, exploration schedule, initial value $\qInit$, learning rate $\learningRate$, train steps $\trainSteps$
      \State Initialize $\hypothesisDFA \gets$ trivial $1$-state rejecting DFA
      \State Initialize $\qTable \gets \qInit$ over $\mdpStates \times \dfaStates_{\hypothesisDFA} \times \mdpActions$
      \State Initialize counterexample set $\causalSample \gets \emptyset$
      \While{fewer than $\trainSteps$ steps have been taken in total}
      \State Reset the environment to $\mdpCommonState_0 \sim \initDist$
      \State Set step $\stepIndex \gets 0$, hypothesis state $\dfaCommonState \gets \dfaInit$ and episode trace $\trace \gets \emptyWord$
      \While{episode not ended and fewer than $\trainSteps$ steps have been taken in total}
      \State Select $\mdpCommonAction_{\stepIndex}$ via $\explorationRate$-greedy on $\qTable(\mdpCommonState_{\stepIndex}, \dfaCommonState, \cdot)$
      \State Execute $\mdpCommonAction_{\stepIndex}$, observe reward $\mdpCommonReward_{\stepIndex}$, next state $\mdpCommonState_{\stepIndex+1}$, and label $\mdpCommonLabel_{\stepIndex}$
      \State $\dfaCommonState' \gets \dfaTrans_{\hypothesisDFA}(\dfaCommonState, \mdpCommonLabel_{\stepIndex})$; append $\mdpCommonLabel_{\stepIndex}$ to $\trace$
      \State $\qTable(\mdpCommonState_{\stepIndex}, \dfaCommonState, \mdpCommonAction_{\stepIndex}) \mathrel{{\gets}} (1-\learningRate)\, \qTable(\mdpCommonState_{\stepIndex}, \dfaCommonState, \mdpCommonAction_{\stepIndex}) + \learningRate \bigl[ \mdpCommonReward_{\stepIndex} + \mdpDiscount \max_{\mdpCommonAction'} \qTable(\mdpCommonState_{\stepIndex+1}, \dfaCommonState', \mdpCommonAction') \bigr]$
      \State $\dfaCommonState \gets \dfaCommonState'$
      \State $\stepIndex \gets \stepIndex + 1$
      \EndWhile
      \If{$\textsc{IsCounterexample}(\trace, \hypothesisDFA, \markerR, \markerA)$} \label{alg:line:isce} \Comment{Section~\ref{sec:counterexample}}
      \State $\causalSample \gets \causalSample \cup \{\trace\}$ \label{alg:line:store}
      \State $\hypothesisDFA' \gets \textsc{SATSynthesize}(\causalSample)$ \label{alg:line:sat} \Comment{Section~\ref{sec:sat_synthesis}}
      \State $\qTable \gets \textsc{ValueTransfer}(\qTable, \hypothesisDFA, \hypothesisDFA', \causalSample)$ \label{alg:line:transfer} \Comment{Section~\ref{sec:value_transfer}}
      \State $\hypothesisDFA \gets \hypothesisDFA'$
      \EndIf
      \EndWhile
      \State \textbf{Return:} $\qTable$, $\hypothesisDFA$
   \end{algorithmic}
\end{algorithm}

\subsection{Bound on the Counterexample Length}
\label{sec:mstar}

The convergence guarantees require episodes long enough to expose every wrong hypothesis that TCIRL may synthesize.
We first bound the length of an attainable counterexample.

Every hypothesis of TCIRL is successor-closed: the initial one accepts nothing, and every synthesized one satisfies Constraint~\ref{itm:sat_successor_closed_acceptance}.
Since the cause DFA is consistent with every sample, the minimal hypothesis has at most $\causeSize$ states (Section~\ref{sec:sat_synthesis}).

Call a successor-closed DFA $\dfa{A}$ wrong if it misclassifies the label sequence $\labelPrefix$ of some unterminated attainable history.
Every wrong DFA has an attainable counterexample, because completeness applies to $\labelPrefix$.
If $\dfa{A}$ rejects $\labelPrefix$ although $\labelPrefix \models \causeFormula$, completeness yields a prefix $\labelPrefix'$ of $\labelPrefix$ and an attainable $\labelPrefix'\mdpCommonLabel'$ with $\markerA \in \mdpCommonLabel'$.
Because $\dfa{A}$ is successor-closed, it also rejects $\labelPrefix'$, making $\labelPrefix'\mdpCommonLabel'$ a counterexample due to soundness.
If $\dfa{A}$ accepts $\labelPrefix$ although $\labelPrefix \not\models \causeFormula$, completeness yields an extension $\labelPrefix''$ of $\labelPrefix$ and an attainable $\labelPrefix''\mdpCommonLabel'$ with $\markerR \in \mdpCommonLabel'$.
Since $\dfa{A}$ also accepts $\labelPrefix''$, this $\labelPrefix''\mdpCommonLabel'$ is a counterexample.

Let $\maxCexLength$ denote the maximum, over all wrong DFAs with at most $\causeSize$ states, of the length of the shortest attainable counterexample.
This constant is finite because there are finitely many such DFAs.

To bound this constant, we represent attainable label sequences by paths in a finite graph.
Each vertex pairs a base state with a cause DFA state, which determines the active phase.

Fix a successor-closed DFA $\causeDFA = (\causeStates, \causeInit, \propAlphabet, \causeTrans, \causeAcc)$ with $\languageOf{\causeDFA} = \languageOf{\causeFormula}$ and $\causeSize$ states, which exists by Section~\ref{sec:background}.
For a cause state $\dfaCommonState \in \causeStates$, write $\causePhase(\dfaCommonState) = 2$ if $\dfaCommonState \in \causeAcc$ and $\causePhase(\dfaCommonState) = 1$ otherwise.

\begin{definition}[Attainability Product]
   \label{def:attainability_product}
   The attainability product of the NMDP is the directed graph with vertex set $\mdpStates \times \causeStates$ and edge labels in $\propAlphabet$.
   It has an edge from $(\mdpCommonState, \dfaCommonState)$ to $(\mdpCommonState', \causeTrans(\dfaCommonState, \mdpCommonLabel))$ labeled $\mdpCommonLabel \in \propAlphabet$ if and only if some action $\mdpCommonAction \in \mdpActions$ satisfies $\phaseKernel{\causePhase(\dfaCommonState)}(\mdpCommonState' \mid \mdpCommonState, \mdpCommonAction) > 0$ and $\mdpLabel(\mdpCommonState, \mdpCommonAction, \mdpCommonState') = \mdpCommonLabel$.
\end{definition}

Each edge records a possible transition in the phase selected by its source vertex.
The next lemma imposes the required starting condition: a base state in the initial distribution's support paired with the cause DFA's initial state.
Paths may repeat vertices.

\begin{lemma}[Attainable Label Sequences]
   \label{lem:attainability_product}
   A label sequence $\trace = \mdpCommonLabel_0 \cdots \mdpCommonLabel_{\seqLength-1}$ is attainable if and only if it labels a path $(\mdpCommonState_0, \causeStateAt{0})$, \ldots, $(\mdpCommonState_{\seqLength}, \causeStateAt{\seqLength})$ in the attainability product with $(\mdpCommonState_0, \causeStateAt{0}) \in \supp(\initDist) \times \{\causeInit\}$.
   On every such path, $\causeStateAt{\stepIndex} = \causeTrans^*(\causeInit, \mdpCommonLabel_0 \cdots \mdpCommonLabel_{\stepIndex-1})$ for $\stepIndex = 0, \ldots, \seqLength$.
\end{lemma}

\begin{proof}
   The cause states along such a path follow $\causeTrans$ from $\causeInit$, which gives the second claim.
   Now let $\mdpHistory_{\seqLength} = \mdpCommonState_0, \mdpCommonAction_0, \ldots, \mdpCommonState_{\seqLength}$ be an attainable history that produces $\trace$, and let $\causeStateAt{\stepIndex} = \causeTrans^*(\causeInit, \mdpCommonLabel_0 \cdots \mdpCommonLabel_{\stepIndex-1})$ for $\stepIndex = 0, \ldots, \seqLength$.
   Since $\languageOf{\causeDFA} = \languageOf{\causeFormula}$, the cause state selects the active phase at every step.
   By Assumption~\ref{asm:two_phase}, each transition of the history therefore has positive probability under the kernel $\phaseKernel{\causePhase(\causeStateAt{\stepIndex})}$.
   Each is thus an edge of the attainability product with the same label, giving a path labeled $\trace$ from $(\mdpCommonState_0, \causeInit) \in \supp(\initDist) \times \{\causeInit\}$.

   Conversely, take such a path and, for each $\stepIndex < \seqLength$, choose an action $\mdpCommonAction_{\stepIndex}$ that witnesses its edge $\stepIndex$.
   The resulting history starts in $\supp(\initDist)$.
   By the second claim, the cause state again selects the active phase at every step, so Assumption~\ref{asm:two_phase} makes each witnessed transition probability positive in the NMDP.
   Hence the chosen actions and the path's base states form an attainable history that produces $\trace$.
   \qed
\end{proof}

Take a path supplied by Lemma~\ref{lem:attainability_product} for a shortest attainable counterexample, and track the hypothesis state along it.
No two proper prefixes can reach the same triple of base, cause and hypothesis states: removing the intervening edges would preserve attainability and the marker exposing the error.
Thus, the number of triples bounds the counterexample length.

\begin{proposition}[Counterexample Length]
   \label{prop:mstar}
   Under Assumptions~\ref{asm:two_phase} and~\ref{asm:markers}, every wrong successor-closed DFA $\hypothesisDFA$ with state set $\dfaStates_{\hypothesisDFA}$ has an attainable counterexample of length at most $|\mdpStates| \cdot \causeSize \cdot |\dfaStates_{\hypothesisDFA}|$.
   Consequently, $\maxCexLength \le |\mdpStates| \cdot \causeSize^2$.
\end{proposition}

\begin{proof}
   Write $\hypothesisDFA = (\dfaStates_{\hypothesisDFA}, \dfaCommonState^{\hypothesisDFA}_{\init}, \propAlphabet, \dfaTrans_{\hypothesisDFA}, \dfaAcc_{\hypothesisDFA})$.
   By Assumption~\ref{asm:markers}, $\hypothesisDFA$ has an attainable counterexample, so let $\trace = \mdpCommonLabel_0 \cdots \mdpCommonLabel_{\seqLength-1}$ be a shortest one.
   A prefix $\labelPrefix\mdpCommonLabel$ of $\trace$ that makes $\trace$ a counterexample is attainable and is itself a counterexample, so $\labelPrefix\mdpCommonLabel = \trace$.
   Hence either $\markerR \in \mdpCommonLabel_{\seqLength-1}$ and $\hypothesisDFA$ accepts $\mdpCommonLabel_0 \cdots \mdpCommonLabel_{\seqLength-2}$, or $\markerA \in \mdpCommonLabel_{\seqLength-1}$ and $\hypothesisDFA$ rejects it.
   For $\stepIndex = 0, \ldots, \seqLength$, let $\dfaCommonState^{\hypothesisDFA}_{\stepIndex} = \dfaTrans_{\hypothesisDFA}^*(\dfaCommonState^{\hypothesisDFA}_{\init}, \mdpCommonLabel_0 \cdots \mdpCommonLabel_{\stepIndex-1})$ be the hypothesis state after the first $\stepIndex$ labels.
   Since $\hypothesisDFA$ accepts $\mdpCommonLabel_0 \cdots \mdpCommonLabel_{\seqLength-2}$ if and only if $\dfaCommonState^{\hypothesisDFA}_{\seqLength-1} \in \dfaAcc_{\hypothesisDFA}$, any attainable label sequence is a counterexample if it ends with $\mdpCommonLabel_{\seqLength-1}$ and $\hypothesisDFA$ reads that label in the state $\dfaCommonState^{\hypothesisDFA}_{\seqLength-1}$.

   Fix a path $(\mdpCommonState_0, \causeStateAt{0})$, \ldots, $(\mdpCommonState_{\seqLength}, \causeStateAt{\seqLength})$ for $\trace$ as in Lemma~\ref{lem:attainability_product}.
   We show that the $\seqLength$ triples $(\mdpCommonState_{\stepIndex}, \causeStateAt{\stepIndex}, \dfaCommonState^{\hypothesisDFA}_{\stepIndex})$ for $\stepIndex = 0, \ldots, \seqLength-1$ are pairwise distinct.
   They lie in $\mdpStates \times \causeStates \times \dfaStates_{\hypothesisDFA}$, so this gives $|\trace| \le |\mdpStates| \cdot \causeSize \cdot |\dfaStates_{\hypothesisDFA}|$.
   Suppose that the triples of two steps $\stepIndex < \stepIndex' \le \seqLength-1$ are equal, and let $\trace' = \mdpCommonLabel_0 \cdots \mdpCommonLabel_{\stepIndex-1} \mdpCommonLabel_{\stepIndex'} \cdots \mdpCommonLabel_{\seqLength-1}$ be $\trace$ without the labels $\mdpCommonLabel_{\stepIndex}, \ldots, \mdpCommonLabel_{\stepIndex'-1}$.
   The path visits the same vertex $(\mdpCommonState_{\stepIndex}, \causeStateAt{\stepIndex}) = (\mdpCommonState_{\stepIndex'}, \causeStateAt{\stepIndex'})$ after $\stepIndex$ and after $\stepIndex'$ edges, so leaving out the edges $\stepIndex, \ldots, \stepIndex'-1$ gives a path labeled $\trace'$ from the same first vertex, and $\trace'$ is attainable by Lemma~\ref{lem:attainability_product}.
   The hypothesis is in the same state $\dfaCommonState^{\hypothesisDFA}_{\stepIndex} = \dfaCommonState^{\hypothesisDFA}_{\stepIndex'}$ after $\stepIndex$ and after $\stepIndex'$ labels, so leaving out the labels $\mdpCommonLabel_{\stepIndex}, \ldots, \mdpCommonLabel_{\stepIndex'-1}$ does not change the state $\dfaCommonState^{\hypothesisDFA}_{\seqLength-1}$ in which it reads the last label $\mdpCommonLabel_{\seqLength-1}$.
   So $\trace'$ is a counterexample, and it is shorter than $\trace$, a contradiction.
   The wrong DFAs in the definition of $\maxCexLength$ have at most $\causeSize$ states, which gives the second claim.
   \qed
\end{proof}

In practice, $\causeSize$ is unknown, and the episode limit is a hyperparameter.

\paragraph{Connection to reward machines.}
This bound is linear in the number of states of the NMDP.
JIRP~\cite{jirp} bounds the corresponding episode limit for reward machines by the number of states of a deterministic automaton of the attainable label sequences.
That automaton comes from a nondeterministic one with $|\mdpStates|$ states and can need $2^{|\mdpStates|}$ states, so the bound is exponential in $|\mdpStates|$.
For reward machines, Neider et al.~\cite{neider_advice2021} state a bound that is polynomial in $|\mdpStates|$ and has the product form of the bound of Proposition~\ref{prop:mstar}.

The same argument applies to reward machines in a labeled MDP.
If a hypothesis reward machine and the true one disagree on an attainable label sequence, they disagree on one whose length is at most the product of the numbers of states of the MDP and of the two machines.
The triples of this argument consist of the MDP state and the states of the two machines.
Equal MDP states keep the shortened sequence attainable, because the dynamics of an MDP do not depend on the label sequence.
Equal states of both machines keep the disagreement at the last transition, because each machine computes the reward of that transition from its state and the last label.

In the proof of Proposition~\ref{prop:mstar}, the marker in the last label plays the role of the true reward, and the shortening keeps that label.
The dynamics of the NMDP depend on the label sequence, so equal base states alone would not keep the shortened sequence attainable, and the triples of that proof also contain the cause state.

\subsection{Theoretical Guarantees}
\label{sec:theory}

We establish two main results: almost-sure recovery of the cause language on the label sequences of all unterminated attainable histories (Theorem~\ref{thm:exact_recovery}) and convergence to an optimal policy (Theorem~\ref{thm:convergence}).
Both proofs rely on an exploration property of $\explorationRate$-greedy policies.

\begin{lemma}[Exploration]
   \label{lem:exploration}
   Under $\explorationRate$-greedy exploration with $\explorationRate \ge \explorationFloor > 0$ throughout training and episode limit at least $\lengthBound$, every $\lengthBound$-attainable label sequence is produced infinitely often almost surely.
\end{lemma}

\begin{proof}
   Fix an $\lengthBound$-attainable sequence $\trace$ and an $\lengthBound$-attainable history that produces it, which fits within the episode limit.
   Under $\explorationRate$-greedy, every action has probability at least $\explorationFloor/|\mdpActions|$ whatever the Q-table and the hypothesis are, so every episode begins with this history with conditional probability at least some $\witnessProb > 0$ given all earlier episodes.
   By the conditional Borel--Cantelli lemma, episodes begin with this history infinitely often a.s., so $\trace$ is produced infinitely often.
   \qed
\end{proof}

\begin{theorem}[Exact Recovery]
   \label{thm:exact_recovery}
   Under Assumptions~\ref{asm:two_phase} and~\ref{asm:markers}, with $\explorationRate$-greedy exploration as in Lemma~\ref{lem:exploration} and episode limit at least $\maxCexLength$, the hypothesis sequence stabilizes almost surely at a final hypothesis $\finalHypothesis$ with $\labelPrefix \in \languageOf{\finalHypothesis} \iff \labelPrefix \models \causeFormula$ for the label sequence $\labelPrefix$ of every unterminated attainable history.
\end{theorem}

\begin{proof}
   First, the hypothesis sequence must stabilize: each re-inference yields a minimal DFA consistent with the accumulated sample, and by soundness a successor-closed DFA with $\causeSize$ states recognizing $\languageOf{\causeFormula}$ is consistent with every sample (Section~\ref{sec:sat_synthesis}).
   Hence every inferred hypothesis has at most $\causeSize$ states.
   Whenever a counterexample $\trace$ is added, it certifies some prefix $\labelPrefix$ that must be rejecting or accepting, while the current hypothesis classifies $\labelPrefix$ the opposite way.
   Every later sample contains that constraint, so that hypothesis never recurs.
   Only finitely many such bounded-size hypotheses exist, so only finitely many re-inferences can occur.
   Second, suppose the final hypothesis $\finalHypothesis$ is wrong.
   It has at most $\causeSize$ states, so it has an attainable counterexample $\trace$ with $|\trace| \le \maxCexLength$.
   By Lemma~\ref{lem:exploration}, a later episode almost surely produces $\trace$, so its label sequence is a counterexample and triggers another re-inference, contradicting stabilization.
   Therefore, $\finalHypothesis$ correctly classifies the label sequence of every unterminated attainable history.
   \qed
\end{proof}

\begin{theorem}[Policy Convergence]
   \label{thm:convergence}
   Under the conditions of Theorem~\ref{thm:exact_recovery} and the standard Q-learning assumptions on learning rates and continued exploration after the final hypothesis update, TCIRL converges a.s.\ to the optimal Q-function of $\productMDP{\finalHypothesis}$.
   The resulting greedy policy achieves the optimal value of the NMDP.
\end{theorem}

\begin{proof}
   After stabilization, $\finalHypothesis$ correctly classifies the label sequence of every unterminated attainable history, so the product MDP $\productMDP{\finalHypothesis}$ has the same transition dynamics as the NMDP on the states that attainable histories reach (Section~\ref{sec:problem_setting}).
   Given a stationary policy $\policy$ on $\productMDP{\finalHypothesis}$, define the induced history-dependent policy on the NMDP by feeding $\policy$ the base state together with the monitored DFA state $\dfaTrans^*(\dfaInit, \mdpLabel(\mdpHistory_{\stepIndex}))$ of the current history $\mdpHistory_{\stepIndex}$.
   Conversely, any history-dependent NMDP policy lifts to a policy on $\productMDP{\finalHypothesis}$ by ignoring the automaton component and using only the projected base-state history.
   Under the natural coupling that shares the initial state and external randomness, corresponding trajectories in the two models then use the same action, transition probabilities, and one-step reward at every time.
   So the induced and lifted policies have the same discounted return, and since a stationary policy is optimal among history-dependent policies in the finite discounted MDP $\productMDP{\finalHypothesis}$, the optimal values of $\productMDP{\finalHypothesis}$ and the NMDP coincide.
   Standard Q-learning convergence from any initial Q-values~\cite{watkins_dayan1992} on the fixed finite MDP $\productMDP{\finalHypothesis}$ therefore yields the optimal Q-function, and the resulting greedy policy achieves the optimal NMDP value.
   \qed
\end{proof}

\section{Optimizations}
\label{sec:optimizations}

This section describes two optimizations of the Q-learning in TCIRL.
Value transfer starts each state of a new hypothesis from the Q-values of a state of the old one (Section~\ref{sec:value_transfer}), and counterfactual updates use each observed transition to update the Q-values of several automaton states (Section~\ref{sec:counterfactual}).

\subsection{Value Transfer}
\label{sec:value_transfer}

Re-synthesizing the hypothesis after a counterexample replaces the automaton states that index the Q-table.
Q-learning, whether or not it observes the cause DFA state, learns on a fixed state space and never discards what it has learned.
A TCIRL agent that restarted from the initial Q-value $\qInit$ after every re-inference would have to relearn its whole Q-table each time, and every re-inference would put it further behind Q-learning that observes the cause DFA state.
Value transfer helps TCIRL catch up by starting every state of the new hypothesis from the Q-values of a state of the old one (line~\ref{alg:line:transfer}).

To choose an old state for each new one, we run the prefixes of the stored counterexamples $\causalSample$ through both hypotheses.
Each prefix then pairs a state of the old hypothesis $\hypothesisDFA$ with a state of the new hypothesis $\hypothesisDFA'$.
For each new state, we choose the old state paired with it by the most prefixes.
We use shared prefixes as a proxy for relevant experience, since states that accept the same label sequences have equal optimal values in the hypothesized models (Lemma~\ref{lem:transfer}) but not necessarily accurate learned estimates.
Formally, TCIRL's transfer heuristic starts each state $\dfaCommonState'$ of $\hypothesisDFA'$ from the Q-values of the state $\dfaCommonState$ of $\hypothesisDFA$ with the most prefixes $\labelPrefix \in \prefixes(\causalSample)$ such that $\dfaTrans_{\hypothesisDFA}^*(\dfaCommonState^{\hypothesisDFA}_{\init}, \labelPrefix) = \dfaCommonState$ and $\dfaTrans_{\hypothesisDFA'}^*(\dfaCommonState^{\hypothesisDFA'}_{\init}, \labelPrefix) = \dfaCommonState'$.

To avoid transferring values from an unrelated old state, we want at least one recorded prefix to reach both the old and the new state in their respective automata.
Every new state has such a match: minimality ensures that every state of $\hypothesisDFA'$ is reached by a prefix in $\prefixes(\causalSample)$.
Otherwise, deleting the states reached by no such prefix and replacing transitions into them with self-loops would preserve the runs on all recorded prefixes and keep the DFA successor-closed.
This would contradict the minimality of $\hypothesisDFA'$.

Under the conditions of Theorem~\ref{thm:convergence}, both cause recovery and policy convergence hold for any initialization of the new Q-table, including resetting every entry to $\qInit$, which we prove in Lemma~\ref{lem:transfer_convergence}.

\begin{lemma}[Value Transfer Preserves Convergence]
   \label{lem:transfer_convergence}
   Under the conditions of Theorem~\ref{thm:convergence}, for any initial Q-values of the states of each new hypothesis, the hypothesis sequence stabilizes as in Theorem~\ref{thm:exact_recovery}, and the Q-values converge to the optimal Q-function of $\productMDP{\finalHypothesis}$.
\end{lemma}

\begin{proof}
   Under $\explorationRate$-greedy exploration, every action has probability at least $\explorationFloor/|\mdpActions|$ whatever the Q-table, so Lemma~\ref{lem:exploration}, and with it the proof of Theorem~\ref{thm:exact_recovery}, does not use the Q-values.
   Only finitely many re-inferences occur, and after the last one TCIRL performs Q-learning on the fixed finite MDP $\productMDP{\finalHypothesis}$, whose dynamics are those of the NMDP on the states that attainable histories reach (Section~\ref{sec:problem_setting}), and Q-learning converges from any initial Q-values~\cite{watkins_dayan1992}.
   \qed
\end{proof}

Section~\ref{sec:transfer_ablation} compares this rule with TCIRL without value transfer and with a transfer that prefers states that accept the same label sequences, as JIRP transfers Q-values between equivalent states of two reward machines~\cite{jirp}.

\subsection{Counterfactual Updates}
\label{sec:counterfactual}

Q-learning for reward machines (QRM)~\cite{reward_machines_icarte} uses an observed base-state transition to update the Q-value of the executed action for every reward-machine state.
For each machine state, the update uses the reward and successor machine state that the reward machine outputs on the observed label.
These counterfactual updates are also used in JIRP~\cite{jirp}.
The principle applies generally to Q-learning with automata, including both TCIRL and the Q-learning baseline with the cause state.
We define a version that accounts for the dependence of the transition dynamics on the automaton state.

Let $\hypothesisDFA$ denote the DFA used by the agent: the current hypothesis for TCIRL or the cause DFA $\causeDFA$ for Q-learning with the cause state.
Suppose the agent takes action $\mdpCommonAction$ in product state $(\mdpCommonState, \dfaCommonState)$ and observes next base state $\mdpCommonState'$, label $\mdpCommonLabel$ and reward $\mdpCommonReward$.
The ordinary Q-update changes $\qTable(\mdpCommonState, \dfaCommonState, \mdpCommonAction)$.
A counterfactual update uses the same observation to update $\qTable(\mdpCommonState, \dfaCommonState', \mdpCommonAction)$ for another automaton state $\dfaCommonState'$, with target
\[
   \mdpCommonReward + \mdpDiscount \max_{\mdpCommonAction'} \qTable(\mdpCommonState', \dfaTrans(\dfaCommonState', \mdpCommonLabel), \mdpCommonAction').
\]
As in Algorithm~\ref{alg:tcirl}, the maximum is $0$ at a terminal base state.

Reward-machine states do not affect the transition dynamics, so the observed base-state transition is a valid sample for updates at every machine state.
In $\productMDP{\hypothesisDFA}$, however, the automaton state selects the transition kernel.
We therefore update only states $\dfaCommonState'$ with $\phaseOf(\dfaCommonState') = \phaseOf(\dfaCommonState)$.
All these states select the same kernel, and the reward depends only on the base-state transition.
Thus, whenever the observation is a valid sample at $(\mdpCommonState, \dfaCommonState)$, it is also valid at $(\mdpCommonState, \dfaCommonState')$.
The set of updated states depends only on the current automaton state, so it is fixed before the transition is sampled.

The cause DFA always identifies the true phase, and the final TCIRL hypothesis does so on every unterminated attainable history (Theorem~\ref{thm:exact_recovery}).
In either case, the additional updates use valid samples and preserve Q-learning convergence under the conditions of Theorem~\ref{thm:convergence}~\cite{tsitsiklis1994}.
If the phase prediction is wrong before the final re-inference, the agent updates all states of the predicted phase with a sample from the wrong kernel.
In a successor-closed DFA, all accepting states accept every label sequence and are therefore equivalent.
A minimal such DFA has at most one accepting state, so the update affects only the visited state when it is accepting, and all rejecting states otherwise.

By default, TCIRL does not use counterfactual updates, so Algorithm~\ref{alg:tcirl} omits them.
Section~\ref{sec:counterfactual_ablation} evaluates them for TCIRL and Q-learning with the cause state.

\section{Experiments}
\label{sec:experiments}

We evaluate TCIRL on two scenarios with hidden cause DFAs and opposite causal effects.
The experiments test whether TCIRL learns to exploit a beneficial phase change and avoid a harmful one.
We first describe the two case studies and the compared variants, then analyze both case studies side by side and isolate the effects of value transfer and counterfactual updates, and finally compare the traffic policies with fixed action cycles and the optimal policy.

\subsection{Case Study 1: Genetic Therapy}
\label{sec:gene_editing}

In the genetic therapy scenario of Section~\ref{sec:introduction}, medication reduces symptoms, while a sequence of gene modifications enables spontaneous recovery.
We examine whether agents learn to trigger this lasting change or rely on repeated medication.
The cause formula is defined as $\causeFormula = \finallyOp(\geneSite{1} \land \nextOp(\finallyOp(\geneSite{2} \land \nextOp(\finallyOp(\geneSite{1} \land \nextOp(\finallyOp(\geneSite{4})))))))$, representing the specific sequence required by the genetic therapy.
The proposition $\geneSite{1}$ appears in a step's label when the step ends in the cell $\geneSite{1}$, and likewise for the other cells.
The agent starts in the center cell, two moves above $\medicineCell$, and a step that ends in $\medicineCell$ administers the medicine.
With probability $0.05$, a move of the agent is replaced by a uniformly random move.
Site $\geneSite{3}$ is a distractor not involved in the cause formula, and the medicine cell $\medicineCell$ affects the dynamics but is not itself causal.
The reward of a step is $1 - \symptomLevel / 9$, with the symptom level $\symptomLevel$ after the step, and $0$ on a step into a terminal state.
To track phase transitions, we define two markers: $\markerR = \effectInc$ (a symptom increase possible only in Phase~1) and $\markerA = \effectRecovery$ (a fall of the symptom level that the medicine does not explain, possible only in Phase~2).

\subsection{Case Study 2: Traffic Signal}
\label{sec:case_study_2}

In traffic, the agent must keep vehicles moving without triggering a signal pattern that causes congestion.
We adapt the traffic signal environment of Partovi Aria and Xu~\cite{pmlr-v288-partovi-aria25a}.
The agent manages a three-intersection road (Fig.~\ref{fig:traffic}) where traffic enters from the leftmost side and flows through successive queues $\queue{0}, \queue{1}, \queue{2} \in \{0, \dots, 8\}$ ($\queue{\queueIndex}$ is the number of vehicles waiting in the $\queueIndex$th queue).
Exactly one intersection is red at any time.
The agent uses the hold action $\actHold$ to maintain the red light's position and the advance action $\actAdvance$ to shift it cyclically to the right, with at least one hold, the mandatory hold, between two advances.
Two advances in a row, without a hold between them, end the episode in a terminal state.
The reward $\mdpCommonReward = 1 - \frac{\max_{\queueIndex} \queue{\queueIndex}}{8}$, with the queues after the step, incentivizes the agent to prevent any single queue from becoming a bottleneck, and a step into a terminal state has reward $0$.

\begin{figure}[tbp]
   \centering
   \begin{tikzpicture}[every node/.style={font=\small}]
      \foreach \i/\x in {0/0, 1/2, 2/4} {
            \node[draw, thick, minimum width=1.4cm, minimum height=0.7cm] (y\i) at (\x, 0) {$\queue{\i}$};
         }
      \draw[->, thick] ($(y0.west)-(1.7, 0)$) -- (y0.west) node[midway, above, font=\scriptsize, text depth=0.3ex] {arrivals};
      \draw[->, thick] (y0.east) -- (y1.west);
      \draw[->, thick] (y1.east) -- (y2.west);
      \draw[->, thick] (y2.east) -- ++(1.7, 0) node[midway, above, font=\scriptsize, text depth=0.3ex] {departures};
      \fill[green!55!black] ($(y0.east)+(0.3, 0.55)$) circle (0.11);
      \fill[red!75!black] ($(y1.east)+(0.3, 0.55)$) circle (0.11);
      \fill[green!55!black] ($(y2.east)+(0.3, 0.55)$) circle (0.11);
   \end{tikzpicture}
   \caption{The road of the traffic signal domain.
      Cars arrive at $\queue{0}$ and leave from $\queue{2}$, and the light after each queue lets its cars move on when it is green.
      Exactly one light is red, here the one after $\queue{1}$, and the advance action moves the red light one light to the right, from the last light back to the first.}
   \label{fig:traffic}
\end{figure}

The label of a step records its action, and the cause formula triggers a phase shift after three consecutive hold-advance pairs ($\actHold, \actAdvance, \actHold, \actAdvance, \actHold, \actAdvance$).
In Phase~1, departures from each green intersection follow $\geomDist{0.3}$ and new arrivals are $\bernoulliDist{0.5}$.
In Phase~2, the two laws swap: departures follow $\bernoulliDist{0.5}$ and arrivals are $\geomDist{0.6}$, increasing congestion.
Here $\geomDist{\geomParam}$ is the number of failures before the first success in a sequence of independent $\bernoulliDist{\geomParam}$ trials, and both geometric draws are capped at the queue capacity $8$.
A car crosses a green light only if the next queue has room, and an arriving car enters $\queue{0}$ only if it has room, so no car is lost.
The rejection marker $\markerR$ fires when $2$ or more cars depart from $\queue{2}$ in a single step, which is only possible under the geometric departure distribution of Phase~1.
Conversely, the acceptance marker $\markerA$ fires when $2$ or more cars enter $\queue{0}$ in a single step, which is only possible under the geometric arrival distribution of Phase~2.

Intuitively, these dynamics model how synchronized light cycles make a specific route through the intersections excessively attractive to outside traffic.
As this road becomes the optimal route, navigation algorithms redirect city traffic onto it, eventually overwhelming the system and causing congestion.
Unlike the genetic therapy case, the optimal policy here must avoid transitioning to Phase~2, where queues increase at a faster rate.

\subsection{Compared Variants and Training}
\label{sec:variants}

In both case studies, we compare five algorithms:
\begin{enumerate}
   \item \textbf{TCIRL}: our method infers $\hypothesisDFA$ via counterexample-driven SAT synthesis using the rejection marker $\markerR$ and the acceptance marker $\markerA$.
   \item \textbf{Q-learning (cause known)}: the agent observes the correct cause DFA state from the start and serves as a full-information reference.
   \item \textbf{Q-learning (cause unknown)}: the agent does not observe any cause DFA state, which shows how much performance drops when the cause is neither inferred nor observed.
   \item \textbf{DQN (4)}: we employ a feedforward deep Q-network (DQN)~\cite{mnih_dqn2015} with a $4$-frame observation history buffer, where a frame is the environment state without the cause state, together with the label of the step into it.
   \item \textbf{DQN (1)}: we limit the DQN to the current frame, a memoryless agent that cannot track the cause.
\end{enumerate}

Counterfactual updates (Section~\ref{sec:counterfactual}) apply equally to TCIRL and Q-learning with the cause state, so the main results show every variant without them, and Section~\ref{sec:counterfactual_ablation} evaluates them separately.

Every variant trains for $500{,}000$ environment steps on each of $50$ independent seeds, and an episode is truncated after $100$ steps.
All agents use the discount factor $\mdpDiscount = 0.99$ and $\explorationRate$-greedy exploration, where $\explorationRate$ decays linearly from $1.0$ to $0.01$: for the tabular agents over the first $100{,}000$ steps in genetic therapy and over all $500{,}000$ steps in traffic, and for the DQN over the first $200{,}000$ steps in both domains.
Each Q-table entry has an update count that records how many times it has been updated.
For its $\updateCount$th update, the entry uses the learning rate $\learningRate_{\updateCount} = \learningRateScale/(\learningRateScale + \updateCount - 1)$, the rescaled linear learning rate of Wainwright~\cite{wainwright_qlearning2019}.
The learning-rate scale $\learningRateScale$ is $100$ in genetic therapy, which is Wainwright's choice $1/(1-\mdpDiscount)$, and $50$ in traffic.
At a re-inference, TCIRL restarts the update count of every entry, so every entry of the new Q-table starts again from $\learningRate_{1} = 1$: its first update replaces the transferred value, which until then guides the choice of actions and the targets of other entries.
These learning rates meet the assumptions of Theorem~\ref{thm:convergence}.
The deep baseline uses Double DQN targets~\cite{vanhasselt_doubledqn2016} with a squared temporal-difference error.
Table~\ref{tab:dqn_hparams} gives its network, optimizer and replay settings.

Every $1{,}000$ steps we pause training and play the current greedy policy on $50$ episodes from a fixed set of evaluation seeds, which gives $501$ evaluation checkpoints over a run, the first at step $0$ on the untrained policy.
Every training seed is tested on the same $50$ evaluation seeds, so differences reflect policy quality rather than environment stochasticity.
The reported evaluation return is the undiscounted sum of rewards within an episode.
The return curves show means over the $50$ training seeds with $95\%$ bootstrap confidence intervals.
They smooth each seed's curve with an exponential moving average of time constant $10{,}000$ steps.
Appendix~\ref{sec:implementation} details the cause inference used in these runs, and Appendix~\ref{sec:hyperparameters} lists their settings.
Appendix~\ref{sec:evaluation} describes the statistical estimates.

\subsection{Returns}
\label{sec:returns}

The return of an episode measures how well a policy keeps the symptoms or the queues low, and Fig.~\ref{fig:returns} shows it for the greedy policy of each variant over training.

\begin{figure}[tbp]
   \centering
   \ref{resultslegend}

   \begin{tabular}{@{}cc@{}}
      \createTrainingPlot
      {results/gene_editing/q_learning-cause_known/eval_mean_cumulative_reward.csv}
      {results/gene_editing/q_learning-cause_unknown/eval_mean_cumulative_reward.csv}
      {results/gene_editing/tcirl-sat/eval_mean_cumulative_reward.csv}
      {results/gene_editing/dqn-cause_unknown_4/eval_mean_cumulative_reward.csv}
      {results/gene_editing/dqn-cause_unknown_1/eval_mean_cumulative_reward.csv}
      {Total Reward per Episode}
      {0}
      {100}
      {paper legend, legend to name=resultslegend, legend style={legend columns=3}}
       &
      \createTrainingPlot
      {results/traffic_signal/q_learning-cause_known/eval_mean_cumulative_reward.csv}
      {results/traffic_signal/q_learning-cause_unknown/eval_mean_cumulative_reward.csv}
      {results/traffic_signal/tcirl-sat/eval_mean_cumulative_reward.csv}
      {results/traffic_signal/dqn-cause_unknown_4/eval_mean_cumulative_reward.csv}
      {results/traffic_signal/dqn-cause_unknown_1/eval_mean_cumulative_reward.csv}
      {Total Reward per Episode}
      {0}
      {75}
      {every axis legend/.code={\let\addlegendentry\dropLegendEntry}}
      \\
      \panelTitle{fig:returns:gene}{Genetic therapy}
       &
      \panelTitle{fig:returns:traffic}{Traffic signal}
   \end{tabular}
   \caption{Smoothed evaluation return of the greedy policy at each of the $501$ checkpoints, as the mean over the $50$ training seeds with its $95\%$ bootstrap interval.}
   \label{fig:returns}
\end{figure}

In genetic therapy (\figpanel{fig:returns}{fig:returns:gene}), TCIRL closely matches the final return of Q-learning with the cause state, with a modest lag during training.
DQN with a buffer of $4$ observation frames ends at a return of $38.1$ against TCIRL's $83.6$.
In contrast, Q-learning without the cause state ends at $32.0$ and the memoryless DQN (1) at $10.6$, so neither learns an effective policy.

In traffic (\figpanel{fig:returns}{fig:returns:traffic}), TCIRL has a modest lead over Q-learning with the cause state during training, partly because its re-inferences restart its update counts and so give it larger learning rates.
Over the first $100{,}000$ steps, the mean evaluation return of TCIRL is $1.7$ $[1.5, 2.0]$ points above that of Q-learning with the cause state, and restarting every update count of Q-learning with the cause state once, at step $10{,}000$ or $20{,}000$, while keeping its Q-values, raises its own by $1.0$ $[0.8, 1.1]$ or $1.2$ $[1.0, 1.3]$ points.
DQN (4) ends at a return of $46.1$ against TCIRL's $64.2$, and the memoryless DQN (1) at $29.3$.
Q-learning without the cause state ends lowest, at $21.2$.

In both domains, TCIRL ends within one point of Q-learning with the cause state, by $-0.1$ $[-0.3, +0.1]$ in genetic therapy and $-0.1$ $[-0.4, +0.2]$ in traffic (Appendix~\ref{sec:evaluation}).

\subsection{Learned Strategies}
\label{sec:strategies}

The trajectories of the final policies show how the variants obtain their returns (Fig.~\ref{fig:trajectories}).

\begin{figure}[tbp]
   \centering
   \ref{resultslegend}

   \begin{tabular}{@{}cc@{}}
      \createTrajectoryPlot
      {results/gene_editing/q_learning-cause_known/trajectories_disease.csv}
      {results/gene_editing/q_learning-cause_unknown/trajectories_disease.csv}
      {results/gene_editing/tcirl-sat/trajectories_disease.csv}
      {results/gene_editing/dqn-cause_unknown_4/trajectories_disease.csv}
      {results/gene_editing/dqn-cause_unknown_1/trajectories_disease.csv}
      {Symptom Level $\symptomLevel$}
      {0}
      {9}
       &
      \createTrajectoryPlot
      {results/traffic_signal/q_learning-cause_known/trajectories_total_queue.csv}
      {results/traffic_signal/q_learning-cause_unknown/trajectories_total_queue.csv}
      {results/traffic_signal/tcirl-sat/trajectories_total_queue.csv}
      {results/traffic_signal/dqn-cause_unknown_4/trajectories_total_queue.csv}
      {results/traffic_signal/dqn-cause_unknown_1/trajectories_total_queue.csv}
      {Total Queue}
      {0}
      {24}
      \\
      \panelTitle{fig:trajectories:gene}{Genetic therapy: symptom level}
       &
      \panelTitle{fig:trajectories:traffic}{Traffic signal: total queue}
   \end{tabular}
   \caption{The symptom level and the total queue within the evaluation episodes of the final checkpoint, measured on the greedy policy.
      Both panels average the episodes still running at each timestep, and an episode can end early in either domain: genetic therapy when medicine is applied on positive harm, traffic when an advance comes during the mandatory hold.
      The mean evaluation episode runs $87.7$ to $100.0$ steps of the $100$ across the five genetic therapy variants and above $95.1$ in traffic.}
   \label{fig:trajectories}
\end{figure}

In genetic therapy, to distinguish lasting recovery from reliance on medication, we examine symptom levels throughout evaluation episodes (\figpanel{fig:trajectories}{fig:trajectories:gene}).
Q-learning without the cause state and DQN (1) initially use a single medication dose to reduce the starting symptom level, on $95.0\%$ and $75.4\%$ of their episodes at step $2$, while DQN with a $4$-frame history does so on only $32.1\%$.
TCIRL and Q-learning with the cause state discover a consistent strategy to modify the genes and trigger the phase transition to the cured state.
They never take the dose at step $2$, and they take one at all in only $3.7\%$ and $4.4\%$ of their episodes.
DQN (4) triggers the cause in only part of its episodes (Section~\ref{sec:cause_activation}), so its symptom level stays higher.
The memoryless DQN (1) leaves the symptom level high, because a policy over the current frame alone has no memory of an earlier visit to $\geneSite{1}$.
Q-learning without the cause state keeps relying on medication, and only half of its seeds reach Phase~2 in most of their episodes (Section~\ref{sec:cause_activation}).

In traffic (\figpanel{fig:trajectories}{fig:trajectories:traffic}), the vehicle counts of TCIRL, Q-learning with the cause state and DQN (4) stay below those of the two memoryless agents from timestep~$11$ on, and the difference grows over most of the rest of the episode.
Avoiding the cause requires breaking the alternation of holds and advances before it completes three hold-advance pairs.
The memoryless DQN (1) cannot count the holds since its last advance, and it triggers the cause in most of its evaluation episodes (Section~\ref{sec:cause_activation}).

\subsection{Outcomes of Individual Seeds}
\label{sec:seed_outcomes}

A mean over $50$ seeds can hide that the seeds of one variant reach different outcomes, so Fig.~\ref{fig:seeds} shows the final return of every seed.

\begin{figure}[tbp]
   \centering
   \begin{tabular}{@{}cc@{}}
      \createSeedScorePlot
      {results/gene_editing/seed_scores.csv}
      {Total Reward per Episode}
      {0}
      {90}
      {}
       &
      \createSeedScorePlot
      {results/traffic_signal/seed_scores.csv}
      {Total Reward per Episode}
      {15}
      {70}
      {}
      \\
      \panelTitle{fig:seeds:gene}{Genetic therapy}
       &
      \panelTitle{fig:seeds:traffic}{Traffic signal}
   \end{tabular}
   \caption{Every training seed's final evaluation return, the last point of its smoothed evaluation curve, one dot per seed and $50$ dots per variant.
      The horizontal position of a dot depends only on its seed index, carries no information, and keeps most dots from overlapping.
      The mean of each column is the endpoint of the matching curve in Fig.~\ref{fig:returns}.}
   \label{fig:seeds}
\end{figure}

In genetic therapy, Q-learning without the cause state shows two groups of training seeds: a dense cluster of low returns and a broad continuum of higher returns (\figpanel{fig:seeds}{fig:seeds:gene}).
The low-return policies rarely complete the gene modification sequence.
Those with higher returns complete it more often, but less reliably than TCIRL or Q-learning with the cause state.
For Q-learning without the cause state, $16$ training seeds have a smoothed cause-completion rate of at most $15\%$ at the end of training.
The seeds of DQN (4) spread over a wide range around their mean, and those of TCIRL and of Q-learning with the cause state lie together at the top.

In traffic, the final returns of the seeds of TCIRL and Q-learning with the cause state overlap, and every TCIRL seed ends above every seed of DQN (4) and of the two memoryless agents (\figpanel{fig:seeds}{fig:seeds:traffic}).

\subsection{Cause Activation}
\label{sec:cause_activation}

The cause activation of a policy is the share of its evaluation episodes that reach Phase~2, which is the goal in genetic therapy and the harmful phase in traffic.
We examine both the final checkpoint's pooled activation and the share of training seeds that reach Phase~2 in most evaluation episodes during training.

\paragraph{Activation in the final evaluation episodes.}
Fig.~\ref{fig:accept} shows cause activation pooled across all training seeds at the final checkpoint.
In both domains, TCIRL reaches Phase~2 in about as many episodes as Q-learning with the cause state.
In genetic therapy, DQN (4) triggers the cause in $65.9\%$ of the evaluation episodes and the memoryless DQN (1) in $1.4\%$.
In traffic, the memoryless DQN (1) triggers the cause in $88.4\%$ of the evaluation episodes, against $41.3\%$ for DQN (4) and $2.8\%$ for TCIRL.

\begin{figure}[tbp]
   \centering
   \createAcceptGroupPlot
   {results/gene_editing/accept_state.csv}
   {(a) Genetic therapy}
   {results/traffic_signal/accept_state.csv}
   {(b) Traffic signal}
   {Episodes in Phase 2 (\%)}
   {112}
   \caption{The share of the final checkpoint's evaluation episodes that reach Phase~2, pooled over all $2{,}500$ of them, with the $95\%$ bootstrap interval over the $50$ training seeds as a whisker.}
   \label{fig:accept}
\end{figure}

\paragraph{Activation across training seeds.}
The pooled activation of $42.2\%$ for Q-learning without the cause state in genetic therapy does not tell us how many training seeds reach Phase~2 in most evaluation episodes.
At each checkpoint, we smooth each seed's cause activation using the same exponential moving average as for the return curves.
We then count the seeds whose smoothed activation exceeds one half and plot their percentage in Fig.~\ref{fig:solved}.

\begin{figure}[tbp]
   \centering
   \ref{solvedlegend}

   \begin{tabular}{@{}cc@{}}
      \createSolvedPlot
      {results/gene_editing/q_learning-cause_known/eval_mean_mostly_in_accept.csv}
      {results/gene_editing/q_learning-cause_unknown/eval_mean_mostly_in_accept.csv}
      {results/gene_editing/tcirl-sat/eval_mean_mostly_in_accept.csv}
      {results/gene_editing/dqn-cause_unknown_4/eval_mean_mostly_in_accept.csv}
      {results/gene_editing/dqn-cause_unknown_1/eval_mean_mostly_in_accept.csv}
      {Seeds in Phase 2 (\%)}
      {paper legend, legend to name=solvedlegend, legend style={legend columns=3}}
       &
      \createSolvedPlot
      {results/traffic_signal/q_learning-cause_known/eval_mean_mostly_in_accept.csv}
      {results/traffic_signal/q_learning-cause_unknown/eval_mean_mostly_in_accept.csv}
      {results/traffic_signal/tcirl-sat/eval_mean_mostly_in_accept.csv}
      {results/traffic_signal/dqn-cause_unknown_4/eval_mean_mostly_in_accept.csv}
      {results/traffic_signal/dqn-cause_unknown_1/eval_mean_mostly_in_accept.csv}
      {Seeds in Phase 2 (\%)}
      {every axis legend/.code={\let\addlegendentry\dropLegendEntry}}
      \\
      \panelTitle{fig:solved:gene}{Genetic therapy}
       &
      \panelTitle{fig:solved:traffic}{Traffic signal}
   \end{tabular}
   \caption{Share of the $50$ training seeds whose policies reach Phase~2 in most evaluation episodes, against training steps.
      The count uses each seed's smoothed episode share, and the bands show $95\%$ Wilson intervals.
      These curves count seeds, whereas Fig.~\ref{fig:accept} pools evaluation episodes.
      Phase~2 is the goal in (\ref{fig:solved:gene}) and the harmful phase in (\ref{fig:solved:traffic}), so a high curve is desirable in the first and a low curve in the second.}
   \label{fig:solved}
\end{figure}

In genetic therapy, the share ends at $100\%$ for TCIRL and Q-learning with the cause state, so every training seed learns a policy that reaches Phase~2 in most episodes.
It ends at $68\%$ for DQN (4), $50\%$ for Q-learning without the cause state and $0\%$ for the memoryless DQN.
The policy of DQN (4) remains unstable: over the last $20$ of the $501$ checkpoints, $39$ of its $50$ seeds drop from at least $90\%$ cause activation to $20\%$ or less at least once.
Every TCIRL seed keeps a cause activation of at least $98\%$ throughout this interval.

In traffic, every seed of TCIRL, of Q-learning with the cause state and of DQN (4) first learns a policy that reaches Phase~2 in most episodes.
Every seed of these two tabular agents then learns to avoid Phase~2, and their curves are at $0\%$ from step~$165{,}000$ on, while the curve of DQN (4) ends at $24\%$ and those of the two memoryless agents at $100\%$.

\subsection{Cause Inference}
\label{sec:cause_inference}

We examine when TCIRL stops revising its hypothesis and whether the final automaton recovers the hidden cause.

\paragraph{Re-inferences.}
As shown in Fig.~\ref{fig:reinferences}, the cumulative number of re-inferences plateaus early in both domains: its mean over $50$ seeds reaches the final value of $2.2$ after $80{,}000$ training steps in genetic therapy and $4.6$ after $29{,}000$ in the traffic signal domain, and no seed re-infers after that.
This early stabilization indicates that every seed settles on its final hypothesis in the initial stage of training.

\paragraph{Learned cause DFAs.}
Fig.~\ref{fig:learned_dfa} shows the hypotheses at the end of training.
In genetic therapy, the final hypothesis of each of the $50$ seeds is the cause DFA itself.
Its $5$ states track the visits to $\geneSite{1}, \geneSite{2}, \geneSite{1}, \geneSite{4}$ in this order, and every other label, including the distractor $\geneSite{3}$ and the medicine cell $\medicineCell$, leaves its state unchanged.

In traffic, every final hypothesis has $7$ states and assigns the same phase as the cause DFA to every unterminated attainable history.
Differences from the cause DFA, and between seeds, occur only on sequences containing two consecutive advances $\actAdvance$.
The second advance terminates the episode, so no later marker can reveal the phase after that step.
These differences do not change the phase information needed for control and are allowed by Theorem~\ref{thm:exact_recovery}, which requires agreement only on unterminated attainable histories.
In \figpanel{fig:learned_dfa}{fig:learned_dfa:traffic}, we show one seed's hypothesis, marking its differing transitions in red and the corresponding cause-DFA transitions in dashed blue.
Both differences concern advances from $\dfaCommonState_{2}$ and $\dfaCommonState_{4}$: an advance enters either state, so a further advance terminates the episode.
The shared transitions count completed hold-advance pairs, and a second consecutive hold before acceptance returns the automaton to $\dfaCommonState_{1}$.

\begin{figure}[tbp]
   \centering
   \createInferenceComparisonPlot
   {results/gene_editing/tcirl-sat/inference_mean_reinferences.csv}
   {results/traffic_signal/tcirl-sat/inference_mean_reinferences.csv}
   {Re-inferences}
   {0}
   {5.5}
   \caption{TCIRL cumulative re-inferences during training (mean over $50$ seeds with its $95\%$ bootstrap interval).}
   \label{fig:reinferences}
\end{figure}

\begin{figure}[tbp]
   \centering
   \begin{tikzpicture}[
         shorten >=1pt, node distance=1.75cm, on grid, >={Stealth[round]},
         every state/.style={minimum size=6.5mm, inner sep=0pt},
         every node/.style={font=\footnotesize},
         initial text={}, initial distance=4mm
      ]
      \node[state, initial] (g0) {$\dfaCommonState_{0}$};
      \node[state] (g1) [right=of g0] {$\dfaCommonState_{1}$};
      \node[state] (g2) [right=of g1] {$\dfaCommonState_{2}$};
      \node[state] (g3) [right=of g2] {$\dfaCommonState_{3}$};
      \node[state, accepting] (g4) [right=of g3] {$\dfaCommonState_{4}$};
      \path[->]
      (g0) edge node[above] {$\geneSite{1}$} (g1)
      (g1) edge node[above] {$\geneSite{2}$} (g2)
      (g2) edge node[above] {$\geneSite{1}$} (g3)
      (g3) edge node[above] {$\geneSite{4}$} (g4)
      (g0) edge[loop below] node[below] {$\neg\geneSite{1}$} ()
      (g1) edge[loop below] node[below] {$\neg\geneSite{2}$} ()
      (g2) edge[loop below] node[below] {$\neg\geneSite{1}$} ()
      (g3) edge[loop below] node[below] {$\neg\geneSite{4}$} ()
      (g4) edge[loop below] node[below] {$\top$} ();
   \end{tikzpicture}

   \panelTitle{fig:learned_dfa:gene}{Genetic therapy: every seed}

   \bigskip
   \begin{tikzpicture}[
         shorten >=1pt, node distance=1.45cm, on grid, >={Stealth[round]},
         every state/.style={minimum size=6.5mm, inner sep=0pt},
         every node/.style={font=\footnotesize},
         learned edge/.style={draw=unknown_red, text=unknown_red, semithick},
         cause edge/.style={draw=oracle_blue, text=oracle_blue, dashed, semithick},
         initial text={}, initial distance=4mm
      ]
      \node[state, initial] (t0) {$\dfaCommonState_{0}$};
      \node[state] (t1) [right=of t0] {$\dfaCommonState_{1}$};
      \node[state] (t2) [right=of t1] {$\dfaCommonState_{2}$};
      \node[state] (t3) [right=of t2] {$\dfaCommonState_{3}$};
      \node[state] (t4) [right=of t3] {$\dfaCommonState_{4}$};
      \node[state] (t5) [right=of t4] {$\dfaCommonState_{5}$};
      \node[state, accepting] (t6) [right=of t5] {$\dfaCommonState_{6}$};
      \path[->]
      (t0) edge node[above] {$\actHold$} (t1)
      (t1) edge node[above] {$\actAdvance$} (t2)
      (t2) edge node[above] {$\actHold$} (t3)
      (t3) edge node[above] {$\actAdvance$} (t4)
      (t4) edge node[above] {$\actHold$} (t5)
      (t5) edge node[above] {$\actAdvance$} (t6)
      (t0) edge[loop below] node[below] {$\actAdvance$} ()
      (t1) edge[loop below] node[below] {$\actHold$} ()
      (t3) edge[bend left=45] node[below, pos=0.35] {$\actHold$} (t1)
      (t5) edge[bend left=45] node[below, pos=0.55] {$\actHold$} (t1)
      (t6) edge[loop above] node[above] {$\actHold, \actAdvance$} ();
      \path[->, learned edge]
      (t2) edge[bend right=35] node[below] {$\actAdvance$} (t3)
      (t4) edge[bend left=45] node[below, pos=0.45] {$\actAdvance$} (t1);
      \path[->, cause edge]
      (t2) edge[bend right=40] node[above] {$\actAdvance$} (t0)
      (t4) edge[bend right=40] node[above] {$\actAdvance$} (t0);
   \end{tikzpicture}

   \panelTitle{fig:learned_dfa:traffic}{Traffic signal: one seed}
   \caption{Final TCIRL hypotheses $\finalHypothesis$, with accepting states drawn as double circles.
      (\ref{fig:learned_dfa:gene})~Edge labels $\neg\geneSite{1}$ and $\top$ denote labels without $\geneSite{1}$ and all labels, respectively.
      (\ref{fig:learned_dfa:traffic})~Solid red edges show differing learned transitions and dashed blue edges their cause-DFA counterparts.}
   \label{fig:learned_dfa}
\end{figure}

\subsection{Ablation of Value Transfer}
\label{sec:transfer_ablation}

TCIRL re-infers its hypothesis a few times early in training (Fig.~\ref{fig:reinferences}), and each re-inference replaces the automaton states that index its Q-table.
Value transfer starts each state of the new hypothesis from the Q-values of the old state paired with it by the most prefixes of the stored counterexamples (Section~\ref{sec:value_transfer}).
To measure what this rule contributes, we compare TCIRL with two variants that differ from it only at a re-inference:
\begin{itemize}
   \item \textbf{TCIRL-None}: TCIRL without value transfer, which resets every Q-value to $\qInit$ and restarts its exploration schedule at every re-inference.
   \item \textbf{TCIRL-Exact}: TCIRL with exact transfer, defined below, which starts each new state from the Q-values of an equivalent old state whenever one exists.
\end{itemize}
TCIRL-None restarts its exploration schedule because the restart helps a reset Q-table, whereas it slows TCIRL and TCIRL-Exact.
All three use the settings of Section~\ref{sec:variants} on the same $50$ seeds, and Q-learning with the cause state again serves as the full-information reference.

\paragraph{Exact transfer.}
We call a state $\dfaCommonState$ of a hypothesis $\hypothesisDFA$ and a state $\dfaCommonState'$ of a hypothesis $\hypothesisDFA'$ equivalent if they accept the same label sequences, that is, if $\dfaTrans_{\hypothesisDFA}^*(\dfaCommonState, \trace) \in \dfaAcc_{\hypothesisDFA} \iff \dfaTrans_{\hypothesisDFA'}^*(\dfaCommonState', \trace) \in \dfaAcc_{\hypothesisDFA'}$ for every $\trace \in \labelSequences$.
The following lemma proves that equivalent states have the same optimal Q-values in the product MDPs of their hypotheses, so Q-values can be transferred between them, as JIRP does between equivalent states of two reward machines~\cite{jirp}.

\begin{lemma}[Equivalent States]
   \label{lem:transfer}
   Let $\dfaCommonState$ and $\dfaCommonState'$ be equivalent states of hypotheses $\hypothesisDFA$ and $\hypothesisDFA'$.
   Then for every base state $\mdpCommonState$ and action $\mdpCommonAction$, the optimal Q-value of $(\mdpCommonState, \dfaCommonState)$ and $\mdpCommonAction$ in $\productMDP{\hypothesisDFA}$ equals that of $(\mdpCommonState, \dfaCommonState')$ and $\mdpCommonAction$ in $\productMDP{\hypothesisDFA'}$.
\end{lemma}

\begin{proof}
   Equivalent states have the same phase, because each accepts the empty sequence exactly when the other does, and for every label $\mdpCommonLabel$ their successors $\dfaTrans_{\hypothesisDFA}(\dfaCommonState, \mdpCommonLabel)$ and $\dfaTrans_{\hypothesisDFA'}(\dfaCommonState', \mdpCommonLabel)$ are again equivalent.
   So from $(\mdpCommonState, \dfaCommonState)$ and $(\mdpCommonState, \dfaCommonState')$, an action $\mdpCommonAction$ leads to every base state $\mdpCommonState'$ with the same probability $\phaseKernel{\phaseOf(\dfaCommonState)}(\mdpCommonState' \mid \mdpCommonState, \mdpCommonAction)$ and the same reward $\mdpReward(\mdpCommonState, \mdpCommonAction, \mdpCommonState')$, and the two successor states again pair $\mdpCommonState'$ with two equivalent automaton states.
   By induction over the iterations, value iteration from the zero function therefore gives $((\mdpCommonState, \dfaCommonState), \mdpCommonAction)$ and $((\mdpCommonState, \dfaCommonState'), \mdpCommonAction)$ the same value for every base state, every action and every pair of equivalent states at once, and it converges to the optimal Q-functions of $\productMDP{\hypothesisDFA}$ and $\productMDP{\hypothesisDFA'}$.
   \qed
\end{proof}

When a re-inference replaces $\hypothesisDFA$ with $\hypothesisDFA'$, exact transfer starts each state $\dfaCommonState'$ of $\hypothesisDFA'$ from the Q-values of a state of $\hypothesisDFA$ that is equivalent to $\dfaCommonState'$.
If $\hypothesisDFA$ has no such state, exact transfer takes the shortest label sequence that leads $\hypothesisDFA'$ from its initial state to $\dfaCommonState'$, the lexicographically first if several exist, and uses the Q-values of the state of $\hypothesisDFA$ that this sequence reaches.

\begin{figure}[tbp]
   \centering
   \ref{transferlegend}

   \begin{tabular}{@{}cc@{}}
      \createTransferPlot
      {results/gene_editing/tcirl-sat/eval_mean_cumulative_reward.csv}
      {results/gene_editing/tcirl-sat-exact_transfer/eval_mean_cumulative_reward.csv}
      {results/gene_editing/tcirl-sat-no_transfer/eval_mean_cumulative_reward.csv}
      {results/gene_editing/q_learning-cause_known/eval_mean_cumulative_reward.csv}
      {0}
      {100}
      {xmax=200000, xtick={0,50000,100000,150000}, paper legend, legend to name=transferlegend, legend style={legend columns=2}}
       &
      \createTransferPlot
      {results/traffic_signal/tcirl-sat/eval_mean_cumulative_reward.csv}
      {results/traffic_signal/tcirl-sat-exact_transfer/eval_mean_cumulative_reward.csv}
      {results/traffic_signal/tcirl-sat-no_transfer/eval_mean_cumulative_reward.csv}
      {results/traffic_signal/q_learning-cause_known/eval_mean_cumulative_reward.csv}
      {0}
      {75}
      {xmax=200000, xtick={0,50000,100000,150000}, every axis legend/.code={\let\addlegendentry\dropLegendEntry}}
      \\
      \panelTitle{fig:transfer:gene}{Genetic therapy}
       &
      \panelTitle{fig:transfer:traffic}{Traffic signal}
   \end{tabular}
   \caption{Smoothed evaluation return of the greedy policy over the first $200{,}000$ training steps for TCIRL, TCIRL-Exact, TCIRL-None and Q-learning with the cause state, as the mean over the $50$ training seeds with its $95\%$ bootstrap interval.
      Table~\ref{tab:transfer} gives the final returns.}
   \label{fig:transfer}
\end{figure}

\begin{table}[tbp]
   \centering
   \caption{Final evaluation return of TCIRL, TCIRL-Exact, TCIRL-None and Q-learning with the cause state, as the mean over $50$ seeds with its $95\%$ bootstrap interval.}
   \label{tab:transfer}
   \createTransferTable{results/gene_editing/transfer_returns.csv}{results/traffic_signal/transfer_returns.csv}
\end{table}

\paragraph{Results.}
Value transfer changes how fast TCIRL learns, not the return at which it ends (Fig.~\ref{fig:transfer} and Table~\ref{tab:transfer}).
TCIRL-None relearns its Q-table and restarts its exploration after every re-inference, so it learns more slowly than TCIRL.
Over the first $100{,}000$ steps, the mean evaluation return of TCIRL is $15.5$ $[13.7, 17.6]$ points above that of TCIRL-None in genetic therapy and $4.5$ $[4.0, 5.0]$ points above it in traffic.
TCIRL-Exact learns about as fast as TCIRL: over the same steps, the mean evaluation return of TCIRL differs from that of TCIRL-Exact by $+0.2$ $[-0.6, +1.0]$ in genetic therapy and exceeds it by $+0.7$ $[+0.4, +0.9]$ in traffic.
The final returns of TCIRL, TCIRL-Exact and TCIRL-None lie within one point of each other.
The final return of TCIRL differs from that of TCIRL-Exact by $-0.2$ $[-0.3, 0.0]$ in genetic therapy and by $+0.1$ $[-0.3, +0.5]$ in traffic, and from that of TCIRL-None by $+0.1$ $[-0.1, +0.3]$ and $+0.2$ $[-0.1, +0.6]$, respectively.

During training, TCIRL lags modestly behind Q-learning with the cause state in genetic therapy and leads it modestly in traffic (Section~\ref{sec:returns}).
In genetic therapy, TCIRL-None lags further behind: it reaches $80$ to $95\%$ of its final return $27{,}000$ to $30{,}000$ steps after Q-learning with the cause state reaches the same shares of its own.

\subsection{Ablation of Counterfactual Updates}
\label{sec:counterfactual_ablation}

We compare TCIRL and Q-learning with the cause state, each with and without counterfactual updates (Section~\ref{sec:counterfactual}).
Each counterfactual update increments the update count of its Q-table entry, which determines that entry's learning rate.
All variants use the settings of Section~\ref{sec:variants} on the same $50$ seeds.

\begin{figure}[tbp]
   \centering
   \ref{counterfactuallegend}

   \begin{tabular}{@{}cc@{}}
      \createCounterfactualPlot
      {results/gene_editing/tcirl-sat/eval_mean_cumulative_reward.csv}
      {results/gene_editing/tcirl-sat-counterfactual/eval_mean_cumulative_reward.csv}
      {results/gene_editing/q_learning-cause_known/eval_mean_cumulative_reward.csv}
      {results/gene_editing/q_learning-cause_known-counterfactual/eval_mean_cumulative_reward.csv}
      {0}
      {100}
      {xmax=500000, xtick={0,100000,200000,300000,400000}, paper legend, legend to name=counterfactuallegend, legend style={legend columns=2}}
      {}
       &
      \createCounterfactualPlot
      {results/traffic_signal/tcirl-sat/eval_mean_cumulative_reward.csv}
      {results/traffic_signal/tcirl-sat-counterfactual/eval_mean_cumulative_reward.csv}
      {results/traffic_signal/q_learning-cause_known/eval_mean_cumulative_reward.csv}
      {results/traffic_signal/q_learning-cause_known-counterfactual/eval_mean_cumulative_reward.csv}
      {0}
      {75}
      {xmax=500000, xtick={0,100000,200000,300000,400000}, every axis legend/.code={\let\addlegendentry\dropLegendEntry}}
      {results/traffic_signal/cycles_optimum.csv}
      \\
      \panelTitle{fig:counterfactual:gene}{Genetic therapy}
       &
      \panelTitle{fig:counterfactual:traffic}{Traffic signal}
   \end{tabular}
   \caption{Smoothed evaluation return of the greedy policy during training for TCIRL and Q-learning with the cause state, each with and without counterfactual updates.
      Curves show the mean over $50$ training seeds with a $95\%$ bootstrap interval.
      The dashed line in the traffic panel marks the expected episode return of the optimal policy of Section~\ref{sec:traffic_cycles} (Table~\ref{tab:cycles}), and Table~\ref{tab:counterfactual} gives the final returns.}
   \label{fig:counterfactual}
\end{figure}

\begin{table}[tbp]
   \centering
   \caption{Final evaluation return of TCIRL and Q-learning with the cause state, each with and without counterfactual updates.
      Values are means over $50$ seeds with $95\%$ bootstrap intervals.}
   \label{tab:counterfactual}
   \createCounterfactualTable{results/gene_editing/counterfactual_returns.csv}{results/traffic_signal/counterfactual_returns.csv}
\end{table}

\paragraph{Results.}
Counterfactual updates improve sample efficiency in both domains (Fig.~\ref{fig:counterfactual}) without changing the optimal policy.
Over the first $100{,}000$ steps, they increase the mean evaluation return by $+4.2$ $[+3.2, +5.2]$ points for TCIRL and $+3.4$ $[+1.6, +5.1]$ points for Q-learning with the cause state in genetic therapy.
In traffic, the corresponding increases are $+14.8$ $[+14.4, +15.1]$ and $+16.2$ $[+15.8, +16.5]$ points.

The number of additional updates depends on how often the agent visits each phase.
Over the first $100{,}000$ steps, a transition updates an average of $1.9$ automaton states for TCIRL and $2.3$ for Q-learning with the cause state in genetic therapy, compared with $5.6$ and $5.7$ in traffic.
During this interval, nearly all traffic transitions occur in Phase~1, so each updates all $6$ rejecting states.
The exploration rate stays above $0.8$, and random actions rarely complete the three hold-advance pairs needed to trigger the cause.
Genetic therapy agents soon learn to trigger the cause and then spend most of their steps in the single accepting state.
Transitions observed there provide no additional updates, which partly explains the smaller gain in sample efficiency.

Counterfactual updates also increase the final evaluation return by $+2.6$ $[+2.4, +2.7]$ points for TCIRL and $+2.3$ $[+2.1, +2.5]$ points for Q-learning with the cause state in genetic therapy.
In traffic, the corresponding increases are $+2.4$ $[+2.1, +2.7]$ and $+2.3$ $[+2.0, +2.6]$ points (Table~\ref{tab:counterfactual}).
In genetic therapy, the final policies learned with counterfactual updates trigger the cause sooner: at timestep~$10$ of the final evaluation episodes, $91.7\%$ and $89.9\%$ of episodes are in Phase~2 for TCIRL and Q-learning with the cause state, respectively, compared with $52.1\%$ and $53.5\%$ without counterfactual updates.
Without counterfactual updates, the genetic therapy agents' returns change by less than $0.1$ points over the last $100{,}000$ steps.
The traffic agents are still learning over this interval: their returns rise by $1.3$ and $1.5$ points without counterfactual updates and by $0.6$ and $0.5$ points with them, for TCIRL and Q-learning with the cause state, respectively.
With counterfactual updates, the agents' final returns remain within one point of each other: the mean return of TCIRL differs from that of Q-learning with the cause state by $+0.1$ $[-0.1, +0.3]$ points in genetic therapy and $0.0$ $[-0.2, +0.2]$ points in traffic.

\subsection{Cycle Benchmarks in Traffic}
\label{sec:traffic_cycles}

TCIRL outperforms DQN (4) in traffic (Section~\ref{sec:returns}).
To test whether DQN's window length alone explains this advantage, we search for policies that achieve higher returns using only information available in its $4$-frame window.
Fixed action cycles provide simple candidates whose memory requirements can be checked directly.
We proceed as follows:
\begin{enumerate}
   \item Enumerate fixed action sequences of at most $12$ steps whose repetition respects the mandatory hold.
   \item Retain cycles reproducible from reset with a $2$- or $4$-frame window, or one frame and the cause state.
         Frames contain only the red-light position and mandatory-hold flag.
         Identical controller inputs must prescribe identical actions.
   \item Select the best tested cycle for each memory using simulated returns, then compute its expected undiscounted return over $100$ steps from reset exactly from the transition probabilities.
         The queues evolve stochastically and determine the reward.
\end{enumerate}

The memory check ensures that the controller can follow the cycle without a separate clock or cycle counter.
For example, ($\actAdvance, \actHold, \actHold$) requires hold after its first hold and advance after its second.
At both points, the red-light position and mandatory-hold flag are identical.
Two frames distinguish them by whether the flag has just cleared or remained clear.

Since the cycle controllers use only part of the agents' observations, their returns are lower bounds on achievable performance with the corresponding memory.
Among the tested cycles, a $4$-frame window achieves nearly the same return as access to the cause state (Table~\ref{tab:cycles}).
TCIRL's final return slightly exceeds that of the cause-state cycle, whereas DQN (4)'s is substantially below that of the $4$-frame cycle.
Thus, better control is possible with DQN (4)'s existing observation window.
The main comparison does not establish that effective traffic control requires explicit cause-state tracking.

\begin{table}[htbp]
   \centering
   \caption{Exact expected returns of the best tested cycles and the optimal policy.
      Learned policies show estimated final means with $95\%$ bootstrap intervals.}
   \label{tab:cycles}
   \createCyclesTable
\end{table}

To compare the learned policies with optimal control, we solve the product MDP with the true cause DFA using the known transition probabilities.
Value iteration gives the optimal policy for the discounted objective.
Its expected undiscounted return over $100$ steps from reset exceeds TCIRL's final return, showing that TCIRL, too, could learn a better policy.

\FloatBarrier

\section{Conclusion}
\label{sec:conclusion}

We presented TCIRL, a framework for jointly learning policies and inferring hidden cause DFAs in environments with sparse stochastic effect signals.
Two atomic marker propositions convert noisy observations into certified counterexample information for a SAT-based cause DFA synthesis loop.
Under a marker soundness and completeness assumption, the hypothesis converges almost surely to a DFA whose language induces the same dynamics as the cause language, and the agent converges to an optimal policy.

On a genetic therapy gridworld with a beneficial phase change and a traffic signal pipeline with a harmful one, TCIRL recovers the correct cause DFA and matches the full-information baseline.

\paragraph{Limitations and future work.}
The marker assumption requires domain expertise to identify phase-exclusive observable events.
The structural constraints model one-shot causation but not cyclic patterns.
Promising directions include automatic marker discovery and an extension to cause automata which can model general regular temporal patterns.

\begin{credits}
   \subsubsection{\ackname}
   This work has been financially supported by the Deutsche Forschungsgemeinschaft, DFG Project number 459419731 and grants NSF CNS 2304863, CNS 2339774 and IIS 2332476.
\end{credits}

\bibliographystyle{splncs04}
\bibliography{references}

\begin{thebibliography}{10}
\providecommand{\url}[1]{\texttt{#1}}
\providecommand{\urlprefix}{URL }
\providecommand{\doi}[1]{https://doi.org/#1}

\bibitem{abadi_rdp2020}
Abadi, E., Brafman, R.I.: Learning and solving regular decision processes. In:
  Proceedings of the Twenty-Ninth International Joint Conference on Artificial
  Intelligence ({IJCAI}) 2020. pp. 1948--1954. ijcai.org (2020).
  \doi{10.24963/IJCAI.2020/270}

\bibitem{bacchus_non_markov1996}
Bacchus, F., Boutilier, C., Grove, A.J.: Rewarding behaviors. In: Proceedings
  of the Thirteenth National Conference on Artificial Intelligence (AAAI) 1996.
  pp. 1160--1167. {AAAI} Press / The {MIT} Press (1996)

\bibitem{bacchus_nmdp1997}
Bacchus, F., Boutilier, C., Grove, A.J.: Structured solution methods for
  non-{M}arkovian decision processes. In: Proceedings of the Fourteenth
  National Conference on Artificial Intelligence (AAAI) 1997. pp. 112--117.
  {AAAI} Press / The {MIT} Press (1997)

\bibitem{brafman_rdp2019}
Brafman, R.I., {De Giacomo}, G.: Regular decision processes: {A} model for
  non-{M}arkovian domains. In: Proceedings of the Twenty-Eighth International
  Joint Conference on Artificial Intelligence ({IJCAI}) 2019. pp. 5516--5522.
  ijcai.org (2019). \doi{10.24963/IJCAI.2019/766}

\bibitem{camacho_non_markov2017}
Camacho, A., Chen, O., Sanner, S., McIlraith, S.A.: Non-{M}arkovian rewards
  expressed in {LTL}: Guiding search via reward shaping. In: Proceedings of the
  Tenth International Symposium on Combinatorial Search ({SoCS-17}). pp.
  159--160 (2017). \doi{10.1609/socs.v8i1.18421}

\bibitem{choi_hidden_mode2001}
Choi, S.P.M., Yeung, D.Y., Zhang, N.L.: Hidden-mode {M}arkov decision processes
  for nonstationary sequential decision making. In: Sun, R., Giles, C.L. (eds.)
  Sequence Learning - Paradigms, Algorithms, and Applications, Lecture Notes in
  Computer Science, vol.~1828, pp. 264--287. Springer (2001).
  \doi{10.1007/3-540-44565-X_12}

\bibitem{degiacomo_vardi2013}
{De Giacomo}, G., Vardi, M.Y.: Linear temporal logic and linear dynamic logic
  on finite traces. In: Proceedings of the 23rd International Joint Conference
  on Artificial Intelligence ({IJCAI}) 2013. pp. 854--860. {IJCAI/AAAI} (2013)

\bibitem{doshi_velez_hip_mdp2016}
Doshi{-}Velez, F., Konidaris, G.D.: Hidden parameter {M}arkov decision
  processes: {A} semiparametric regression approach for discovering latent task
  parametrizations. In: Proceedings of the Twenty-Fifth International Joint
  Conference on Artificial Intelligence ({IJCAI}) 2016. pp. 1432--1440.
  {IJCAI/AAAI} Press (2016)

\bibitem{actual_causality}
Halpern, J.Y.: Actual Causality. The MIT Press (2016).
  \doi{10.7551/mitpress/10809.001.0001}

\bibitem{hasanbeig_ltl_rl2019}
Hasanbeig, M., Kantaros, Y., Abate, A., Kroening, D., Pappas, G.J., Lee, I.:
  Reinforcement learning for temporal logic control synthesis with
  probabilistic satisfaction guarantees. In: 58th {IEEE} Conference on Decision
  and Control ({CDC}) 2019. pp. 5338--5343. {IEEE} (2019).
  \doi{10.1109/CDC40024.2019.9028919}

\bibitem{heule_verwer2010}
Heule, M.J.H., Verwer, S.: Exact {DFA} identification using {SAT} solvers. In:
  Proceedings of the 10th International Colloquium Conference on Grammatical
  Inference: Theoretical Results and Applications. Lecture Notes in Computer
  Science, vol.~6339, pp. 66--79. Springer-Verlag, Berlin, Heidelberg (2010).
  \doi{10.1007/978-3-642-15488-1_7}

\bibitem{kaelbling_pomdp1998}
Kaelbling, L.P., Littman, M.L., Cassandra, A.R.: Planning and acting in
  partially observable stochastic domains. Artif. Intell.  \textbf{101}(1--2),
  99--134 (1998). \doi{10.1016/S0004-3702(98)00023-X}

\bibitem{littman_memoryless1994}
Littman, M.L.: Memoryless policies: Theoretical limitations and practical
  results. In: Proceedings of the Third International Conference on Simulation
  of Adaptive Behavior: From Animals to Animats 3. pp. 238--245. MIT Press,
  Cambridge, MA, USA (1994). \doi{10.7551/mitpress/3117.003.0041}

\bibitem{lu_causal_rl2018}
Lu, C., Sch{\"{o}}lkopf, B., Hern{\'{a}}ndez{-}Lobato, J.M.: Deconfounding
  reinforcement learning in observational settings. CoRR
  \textbf{abs/1812.10576} (2018), \url{http://arxiv.org/abs/1812.10576}

\bibitem{mnih_dqn2015}
Mnih, V., Kavukcuoglu, K., Silver, D., Rusu, A.A., Veness, J., Bellemare, M.G.,
  Graves, A., Riedmiller, M., Fidjeland, A.K., Ostrovski, G., Petersen, S.,
  Beattie, C., Sadik, A., Antonoglou, I., King, H., Kumaran, D., Wierstra, D.,
  Legg, S., Hassabis, D.: Human-level control through deep reinforcement
  learning. Nature  \textbf{518}(7540),  529--533 (2015).
  \doi{10.1038/NATURE14236}

\bibitem{neider_advice2021}
Neider, D., Gaglione, J.R., Gavran, I., Topcu, U., Wu, B., Xu, Z.:
  Advice-guided reinforcement learning in a non-{M}arkovian environment. In:
  Proceedings of the Thirty-Fifth AAAI Conference on Artificial Intelligence
  ({AAAI}) 2021. pp. 9073--9080. {AAAI} Press (2021).
  \doi{10.1609/aaai.v35i10.17096}

\bibitem{pmlr-v288-partovi-aria25a}
Partovi~Aria, H., Xu, Z.: Mining causal signal temporal logic formulas for
  efficient reinforcement learning with temporally extended tasks. In:
  Proceedings of the International Conference on Neuro-symbolic Systems.
  vol.~288, pp. 524--542. PMLR (2025),
  \url{https://proceedings.mlr.press/v288/partovi-aria25a.html}

\bibitem{pearl_causality2009}
Pearl, J.: Causality: Models, Reasoning, and Inference. Cambridge University
  Press, Cambridge, UK, 2nd edn. (2009). \doi{10.1017/CBO9780511803161}

\bibitem{sadigh_specifications2014}
Sadigh, D., Kim, E.S., Coogan, S., Sastry, S.S., Seshia, S.A.: A learning based
  approach to control synthesis of {M}arkov decision processes for linear
  temporal logic specifications. In: 53rd {IEEE} Conference on Decision and
  Control ({CDC}) 2014. pp. 1091--1096. {IEEE} (2014).
  \doi{10.1109/CDC.2014.7039527}

\bibitem{smallwood_sondik1973}
Smallwood, R.D., Sondik, E.J.: The optimal control of partially observable
  {M}arkov processes over a finite horizon. Oper. Res.  \textbf{21}(5),
  1071--1088 (1973). \doi{10.1287/OPRE.21.5.1071}

\bibitem{sutton_barto2018}
Sutton, R.S., Barto, A.G.: Reinforcement Learning: An Introduction. {MIT}
  Press, 2nd edn. (2018)

\bibitem{reward_machines_icarte}
{Toro Icarte}, R., Klassen, T.Q., Valenzano, R.A., McIlraith, S.A.: Reward
  machines: Exploiting reward function structure in reinforcement learning. J.
  Artif. Intell. Res.  \textbf{73},  173--208 (2022).
  \doi{10.1613/JAIR.1.12440}

\bibitem{tsitsiklis1994}
Tsitsiklis, J.N.: Asynchronous stochastic approximation and {Q}-learning.
  Machine Learning  \textbf{16}(3),  185--202 (1994). \doi{10.1007/BF00993306}

\bibitem{ulyantsev_symmetry2015}
Ulyantsev, V., Zakirzyanov, I., Shalyto, A.: {BFS}-based symmetry breaking
  predicates for {DFA} identification. In: Proceedings of the 9th International
  Conference on Language and Automata Theory and Applications ({LATA}). Lecture
  Notes in Computer Science, vol.~8977, pp. 611--622. Springer (2015).
  \doi{10.1007/978-3-319-15579-1_48}

\bibitem{vanhasselt_doubledqn2016}
{van Hasselt}, H., Guez, A., Silver, D.: Deep reinforcement learning with
  double {Q}-learning. In: Proceedings of the Thirtieth AAAI Conference on
  Artificial Intelligence. pp. 2094--2100. AAAI Press (2016).
  \doi{10.1609/aaai.v30i1.10295}

\bibitem{wainwright_qlearning2019}
Wainwright, M.J.: Stochastic approximation with cone-contractive operators:
  Sharp $\ell_\infty$-bounds for {Q}-learning. CoRR  \textbf{abs/1905.06265}
  (2019), \url{http://arxiv.org/abs/1905.06265}

\bibitem{watkins_dayan1992}
Watkins, C.J.C.H., Dayan, P.: Technical note: Q-learning. Machine Learning
  \textbf{8}(3--4),  279--292 (1992). \doi{10.1023/A:1022676722315}

\bibitem{wilson1927}
Wilson, E.B.: Probable inference, the law of succession, and statistical
  inference. J. Amer. Statist. Assoc.  \textbf{22}(158),  209--212 (1927).
  \doi{10.1080/01621459.1927.10502953}

\bibitem{jirp}
Xu, Z., Gavran, I., Ahmad, Y., Majumdar, R., Neider, D., Topcu, U., Wu, B.:
  Joint inference of reward machines and policies for reinforcement learning.
  In: Proceedings of the Thirtieth International Conference on Automated
  Planning and Scheduling ({ICAPS}) 2020. pp. 590--598. {AAAI} Press (2020).
  \doi{10.1609/icaps.v30i1.6756}

\bibitem{xu_active_fra2021}
Xu, Z., Wu, B., Ojha, A., Neider, D., Topcu, U.: Active finite reward automaton
  inference and reinforcement learning using queries and counterexamples. In:
  Machine Learning and Knowledge Extraction: 5th International Cross-Domain
  Conference, CD-MAKE. Lecture Notes in Computer Science, vol. 12844, pp.
  115--135. Springer-Verlag, Berlin, Heidelberg (2021).
  \doi{10.1007/978-3-030-84060-0_8}

\bibitem{zeng_causal_rl_survey2025}
Zeng, Y., Cai, R., Sun, F., Huang, L., Hao, Z.: A survey on causal
  reinforcement learning. {IEEE} Trans. Neural Netw. Learn. Syst.
  \textbf{36}(4),  5942--5962 (2025). \doi{10.1109/TNNLS.2024.3403001}

\end{thebibliography}

\clearpage

\appendix
\renewcommand{\theHsection}{\Alph{section}}
\section{Implementation of Cause Inference}
\label{sec:implementation}

TCIRL's implementation restricts the inference alphabet and batches counterexamples before calling the SAT solver.
This appendix describes these choices and their effect on the synthesis calls reported in Section~\ref{sec:cause_inference}.

\paragraph{Labels and trace reduction.}
The hypothesis DFA and the SAT search use labels restricted to the propositions that can appear in a cause formula.
In genetic therapy, at most one cell proposition appears in each label, giving six symbols in total, including the empty label.

In traffic, a single Boolean proposition records whether the agent chose the advance action $\actAdvance$.
This proposition appears in the label for $\actAdvance$.
The label for the hold action $\actHold$ is empty.
The hypothesis DFA therefore has two input symbols, one for each action.
In both domains, the markers are read from the full label but do not appear in the restricted labels, since they appear in neither cause formula.

The user can declare a symbol that never changes the cause state.
The search removes that symbol from the traces and adds its self-loops back to the inferred DFA.
Genetic therapy declares the empty label of the unlabeled cells.
Traffic declares no symbol, because its alternating pattern $\actHold, \actAdvance, \actHold, \actAdvance, \actHold, \actAdvance$ contains hold, the empty label there.
Removing the symbol does not change the phase of any prefix, so each marker still certifies the prefix in front of it.
The value transfer of Section~\ref{sec:value_transfer} counts the distinct prefixes of these reduced traces, so prefixes that differ only in the removed symbol count once.
A reduced prefix leads every hypothesis to the same state as each full prefix that reduces to it, because every hypothesis loops on that symbol.

\paragraph{SAT search.}
Symmetry-breaking clauses~\cite{ulyantsev_symmetry2015} admit only one numbering of the states of each DFA.
In this numbering, a breadth-first search from the initial state visits the states in order.
The clauses leave the smallest satisfiable state count unchanged.
The search tries at most $10$ states.
Training stops with an error if no size up to this bound is satisfiable.

\paragraph{Batched inference.}
The training loop steps the vectorized environment in batches of $1{,}000$ steps in both domains.
The implementation checks for new counterexamples every $10{,}000$ steps in genetic therapy and every $1{,}000$ in traffic.
If the episodes since the previous check added a counterexample, it re-synthesizes the hypothesis once from the whole sample.
This batches the per-counterexample updates of Algorithm~\ref{alg:tcirl}.

Every stored trace keeps the markers it was observed with.
Batching delays synthesis without changing this evidence.
The re-inferences counted in Section~\ref{sec:cause_inference} are these syntheses.

We measure the workload by the counterexample episodes collected and the syntheses run per seed.
In genetic therapy, each seed collects $75.3$ counterexample episodes and runs $2.2$ syntheses on average over the $50$ seeds.
In traffic, the corresponding means are $35.2$ counterexample episodes and $4.6$ syntheses.

\section{Statistical Methods}
\label{sec:evaluation}

The variants share one bank of $50$ training seeds: seed $\seedIndex$ of every variant starts from the same state of the random number generator.
The three tabular variants derive their environment randomness from it identically, while the two deep variants first derive a separate stream for the initialization of their networks.
A difference between two variants is the mean of the $50$ differences between their seeds of the same index.
Every $95\%$ interval that the paper prints for a mean over the seeds is a percentile bootstrap interval from $50{,}000$ resamples of the $50$ training seeds, and the bands of the curves use $2{,}000$ resamples at each checkpoint.
The bands of Fig.~\ref{fig:solved} are instead $95\%$ Wilson intervals~\cite{wilson1927} for the probability that a new training seed reaches Phase~2 in most evaluation episodes: $[92.9\%, 100\%]$ when all $50$ seeds do and $[0\%, 7.1\%]$ when none does.

A seed's final return is the last point of its smoothed evaluation curve, so it averages the seed's greedy returns over its last checkpoints.
The smoothing matters mainly for the DQN variants: over the last $20$ of the $501$ checkpoints, the standard deviation of a seed's greedy return has a median over the seeds of $5.1$ to $21.0$ points for the two DQN variants in the two domains, and of at most $0.6$ for TCIRL.
If each seed is scored by its final checkpoint alone, TCIRL still ends within one point of Q-learning with the cause state and above DQN (4) in both domains, while the returns of each DQN variant spread more widely across its seeds.

\section{Hyperparameters}
\label{sec:hyperparameters}

This appendix lists the training and environment settings used in Section~\ref{sec:experiments}.
Table~\ref{tab:shared_hparams} lists what all five variants share, Table~\ref{tab:tabular_hparams} the three tabular variants, and Table~\ref{tab:dqn_hparams} the deep baseline.

\subsection{Shared Training Settings}
\label{sec:hyperparameters:shared}

\begin{table}[htbp]
   \centering
   \caption{Settings shared by all five variants in both case studies.}
   \label{tab:shared_hparams}
   \footnotesize
   \setlength{\tabcolsep}{3pt}
   \begin{tabular}{@{}l>{\raggedright\arraybackslash}p{0.28\columnwidth}>{\raggedright\arraybackslash}p{0.28\columnwidth}@{}}
      \toprule
      Hyperparameter                    & Genetic therapy & Traffic     \\
      \midrule
      Discount $\mdpDiscount$           & $0.99$          & $0.99$      \\
      Training steps                    & $500{,}000$     & $500{,}000$ \\
      Episode limit                     & $100$ steps     & $100$ steps \\
      Training seeds                    & $50$            & $50$        \\
      Greedy evaluation episodes / seed & $50$            & $50$        \\
      Evaluation checkpoints            & $501$           & $501$       \\
      \bottomrule
   \end{tabular}
\end{table}

\subsection{Algorithm Settings}
\label{sec:hyperparameters:algorithms}

The tabular variants share their learning rule, while the two DQN variants share their network and optimizer settings.
\begin{table}[htbp]
   \centering
   \caption{The three tabular variants: TCIRL, Q-learning with the cause state observable, and Q-learning without it.
      The three differ in what they observe, and TCIRL also transfers its Q-values at a re-inference.}
   \label{tab:tabular_hparams}
   \footnotesize
   \setlength{\tabcolsep}{3pt}
   \begin{tabular}{@{}l>{\raggedright\arraybackslash}p{0.28\columnwidth}>{\raggedright\arraybackslash}p{0.28\columnwidth}@{}}
      \toprule
      Hyperparameter                                                        & Genetic therapy                                                                                                            & Traffic               \\
      \midrule
      $\explorationRate$ from $1.0$ to $0.01$, linearly over                & first $100{,}000$ steps                                                                                                    & all $500{,}000$ steps \\
      Learning rate $\learningRate_{\updateCount}$ of update $\updateCount$ & \multicolumn{2}{l}{$\learningRateScale/(\learningRateScale + \updateCount - 1)$ per entry~\cite{wainwright_qlearning2019}}                         \\
      Learning-rate scale $\learningRateScale$                              & $100$                                                                                                                      & $50$                  \\
      Q initialization                                                      & $25$                                                                                                                       & $60$                  \\
      \midrule
      \multicolumn{3}{l}{\textit{TCIRL only}}                                                                                                                                                                                    \\
      Cause inference                                                       & \multicolumn{2}{l}{counterexample-driven SAT synthesis}                                                                                            \\
      Largest hypothesis searched                                           & \multicolumn{2}{l}{$10$ states}                                                                                                                    \\
      Symmetry breaking                                                     & \multicolumn{2}{l}{BFS predicates~\cite{ulyantsev_symmetry2015}}                                                                                   \\
      Declared self-loop symbol                                             & unlabeled cell                                                                                                             & none                  \\
      Transferred at a re-inference                                         & \multicolumn{2}{l}{Q-values (Section~\ref{sec:value_transfer})}                                                                                    \\
      Restarted at a re-inference                                           & \multicolumn{2}{l}{update counts}                                                                                                                  \\
      \bottomrule
   \end{tabular}
\end{table}

The deep baseline uses a constant optimizer learning rate, while tabular learning rates decrease with each entry's update count.

\begin{table}[htbp]
   \centering
   \caption{The deep cause-blind baseline.
      One frame is the observed state of the environment, without the cause state, followed by the label of the transition into it, so a $\numFrames$-frame window is the bounded-memory stand-in for the cause state, which this variant does not observe.}
   \label{tab:dqn_hparams}
   \footnotesize
   \setlength{\tabcolsep}{3pt}
   \begin{tabular}{@{}l>{\raggedright\arraybackslash}p{0.55\columnwidth}@{}}
      \toprule
      Hyperparameter                                         & Value                                      \\
      \midrule
      Architecture                                           & MLP, two ReLU layers of $128$, linear head \\
      Observation history                                    & $4$ frames and $1$ frame                   \\
      $\explorationRate$ from $1.0$ to $0.01$, linearly over & first $200{,}000$ steps                    \\
      Optimizer                                              & Adam, constant rate $0.001$                \\
      Replay buffer                                          & $100{,}000$ transitions                    \\
      Minibatch                                              & $64$                                       \\
      Warm-up before the first update                        & $1{,}000$ environment steps                \\
      Gradient steps                                         & one per environment step after the warm-up \\
      Target network sync                                    & every $1{,}000$ gradient steps             \\
      Temporal-difference target                             & Double DQN~\cite{vanhasselt_doubledqn2016} \\
      Loss                                                   & squared error, no gradient clipping        \\
      \bottomrule
   \end{tabular}
\end{table}

\subsection{Environment Settings}
\label{sec:hyperparameters:environments}

\begin{table}[htbp]
   \centering
   \caption{Environment settings for the two case studies.}
   \label{tab:env_hparams}
   \footnotesize
   \setlength{\tabcolsep}{3pt}
   \begin{tabular}{@{}l>{\raggedright\arraybackslash}p{0.53\columnwidth}@{}}
      \toprule
      \multicolumn{2}{l}{\textit{Genetic therapy}}                                                                                      \\
      \midrule
      Grid                            & $5 \times 5$ walkable cells inside a wall ring                                                  \\
      Movement slip probability       & $0.05$                                                                                          \\
      Symptom level $\symptomLevel$   & $\{0, \dots, 9\}$, initially $5$                                                                \\
      Medicine dose                   & removes $5$ levels, and another dose within $9$ steps ends the episode                          \\
      Phase~1 symptom step            & $+1$ with probability $0.75$                                                                    \\
      Phase~2 symptom step            & $-1$ with probability $0.75$                                                                    \\
      Reward                          & $1 - \symptomLevel / 9$, and $0$ on a step into a terminal state                                \\
      Cause                           & the ordered visits $\geneSite{1}, \geneSite{2}, \geneSite{1}, \geneSite{4}$;                    \\
                                      & $5$ DFA states                                                                                  \\
      $\markerR$ / $\markerA$         & a symptom rise / a spontaneous recovery                                                         \\
      \midrule
      \multicolumn{2}{l}{\textit{Traffic}}                                                                                              \\
      \midrule
      Queues                          & three, $\queue{\queueIndex} \in \{0, \dots, 8\}$, initially $2$                                 \\
      Actions                         & hold $\actHold$, advance $\actAdvance$                                                          \\
      Mandatory hold after an advance & $1$ step; an advance during it ends the episode                                                 \\
      Phase~1 departures / arrivals   & $\geomDist{0.3}$ / $\bernoulliDist{0.5}$                                                        \\
      Phase~2 departures / arrivals   & $\bernoulliDist{0.5}$ / $\geomDist{0.6}$                                                        \\
      Reward                          & $1 - \max_{\queueIndex} \queue{\queueIndex} / 8$, and $0$ on a step into a terminal state       \\
      Cause                           & three hold-advance pairs $\actHold, \actAdvance, \actHold, \actAdvance, \actHold, \actAdvance$; \\
                                      & $7$ DFA states                                                                                  \\
      $\markerR$ / $\markerA$         & $2$ or more departures / $2$ or more arrivals                                                   \\
      \bottomrule
   \end{tabular}
\end{table}

In Table~\ref{tab:env_hparams} and in Section~\ref{sec:case_study_2}, $\geomDist{\geomParam}$ is the number of failures before the first success in a sequence of independent $\bernoulliDist{\geomParam}$ trials, capped at the queue capacity $8$.
Without the cap its mean is $(1-\geomParam)/\geomParam$, and with it $\left((1-\geomParam)/\geomParam\right)\left(1-(1-\geomParam)^{8}\right)$, which is $2.20$ cars at $\geomParam = 0.3$ and $0.67$ at $\geomParam = 0.6$.

\end{document}